\documentclass[10pt,journal,compsoc]{IEEEtran}

\usepackage{bbding}
\usepackage{graphicx}
\usepackage{amssymb}
\usepackage{amsmath} 
\usepackage{mathrsfs}
\usepackage{mathtools}
\usepackage{multirow}
\usepackage{multicol}
\usepackage{makecell}
\usepackage{caption}
\usepackage{footmisc}
\usepackage[compact]{titlesec}
\usepackage{subcaption}
\usepackage[T1]{fontenc}
\usepackage{hyperref}
\usepackage{arydshln} 
\usepackage{bbm}
\usepackage{xcolor}
\usepackage{diagbox}
\usepackage{algorithm}
\usepackage[noend]{algorithmic}

\DeclareMathOperator*{\argmin}{arg\,min}

\hypersetup{
    colorlinks=true,
    linkcolor=blue,
    filecolor=magenta,      
    urlcolor=cyan,
}

\usepackage{enumitem}
\usepackage[english]{babel}
\usepackage{amsthm}
\newtheorem{theorem}{Theorem}[section]
\newtheorem{lemma}[theorem]{Lemma}

\ifCLASSOPTIONcompsoc
  \usepackage[nocompress]{cite}
\else
  \usepackage{cite}
\fi

\ifCLASSINFOpdf
\else
\fi

\begin{document}
\newcommand{\workname}{DuaL} 
\title{Dual Adversarial Alignment for Realistic Support-Query Shift Few-shot Learning}

\author{Siyang Jiang,  Rui Fang, Hsi-Wen Chen, Wei Ding, and Ming-Syan Chen,~\IEEEmembership{Fellow,~IEEE}
\IEEEcompsocitemizethanks{\IEEEcompsocthanksitem S. Jiang is with the School of Mathematics and Statistics, Huizhou University, Huizhou, China. He was also with the Department
of Electrical and Engineering, National Taiwan University, Taipei, TW. \protect\\ 
E-mail: syjiang@hzu.edu.cn.
\IEEEcompsocthanksitem  R. Fnag, H.-W. Chen, W. Ding, M.-S. Chen are with the Department
of Electrical and Engineering, National Taiwan University, Taipei, TW.\protect\\ 
E-mail: \{rfang, hwchen, wding\}@arbor.ee.ntu.edu.tw,mschen@ntu.edu.tw.
}
\thanks{Manuscript received April 19, 2005; revised August 26, 2015.}}

\markboth{Journal of \LaTeX\ Class Files,~Vol.~14, No.~8, August~2015}%
{Shell \MakeLowercase{\textit{et al.}}: Bare Demo of IEEEtran.cls for Computer Society Journals}
\IEEEtitleabstractindextext{%
\begin{abstract}

Support-query shift few-shot learning aims to classify unseen examples (query set) to labeled data (support set) based on the learned embedding in a low-dimensional space under a distribution shift between the support set and the query set.  However, in real-world scenarios the shifts are usually unknown and varied, making it difficult to estimate in advance. Therefore, in this paper, we propose a novel but more difficult challenge, RSQS, focusing on Realistic Support-Query Shift few-shot learning. The key feature of RSQS is that the individual samples in a meta-task are subjected to \textit{multiple} distribution shifts in each meta-task. In addition, we propose a unified adversarial feature alignment method called \textit{\textbf{DU}al adversarial \textbf{AL}ignment framework (DuaL)} to relieve RSQS from two aspects, i.e., \textit{inter-domain bias} and \textit{intra-domain variance}. On the one hand, for the inter-domain bias, we corrupt the original data in advance and use the synthesized perturbed inputs to train the repairer network by minimizing distance in the feature level. On the other hand, for intra-domain variance,  we proposed a generator network to synthesize hard, i.e., less similar, examples from the support set in a self-supervised manner and introduce regularized optimal transportation to derive a smooth optimal transportation plan. Lastly, a benchmark of RSQS is built with several state-of-the-art baselines among three datasets (CIFAR100, mini-ImageNet, and Tiered-Imagenet). Experiment results show that \workname significantly outperforms the state-of-the-art methods in our benchmark.
\end{abstract}

}

\maketitle
\IEEEdisplaynontitleabstractindextext
\IEEEpeerreviewmaketitle

\IEEEraisesectionheading{\section{Introduction}\label{sec:introduction}}

\IEEEPARstart{D}{eep} learning algorithms have been successful in a variety of computer vision tasks, such as  image classification~\cite{caldas2018leaf}, object detection~\cite{shuai2022balancefl}, and instance segmentation~\cite{ye2022licam}. However, these methods typically require a large amount of labeled data to achieve high accuracy, which is usually challenging and time-consuming to obtain in real-world applications~\cite{garcia2017few,jiang2022pgada}. To address this issue, the field of \emph{Few-shot Learning (FSL)}~\cite{vinyals2016matching,snell2017prototypical} has gained attention as a way to learn from a limited number of examples with supervised information and then effectively adapt it in testing phase\cite{finn2017model,bennequin2021bridging}. The main challenge in FSL is that the training and testing labels are different, meaning the labels in the testing phase are unseen before.

To address this issue, FSL methods classify the query set (unlabeled data) into the support sets (labeled data) by comparing the similarity of the learned embeddings, rather than training a classifier for each testing label. For example, MatchingNet~\cite{vinyals2016matching} classifies the query example to the label of the most similar example in the support set. ProtoNet~\cite{snell2017prototypical} assumes a single prototype, e.g., centroid, of each class in the embedding space and then compares the distances between the query example and the prototype representation. Recent studies in FSL aim to estimate data distributions~\cite{song2022comprehensive} by enhancing model robustness through techniques such as data augmentation~\cite{zhong2020random,devries2017improved}, transfer learning~\cite{nakamura2019revisiting,cai2020cross}, and multi-modal learning~\cite{li2019large,wang2020generalizing}. Data augmentation methods approximate the true data distribution using pre-defined rules~\cite{li2021learning,gao2018low} or objectives~\cite{ma2020metacgan,kim2019variational}. Transfer and multi-modal learning focus on using knowledge from other domains or modalities to improve performance~\cite{li2019large, schwartz2022baby}.

As shown in Fig~\ref{fig:flow_1}, compared to the conventional FSL, the novel few-shot learning problem is considered i.e., \emph{Support-Query Shift}~\cite{bennequin2021bridging,aimen2022adversarial}, denoting the distribution shift between the support set and query set embeddings. This shift occurs in the learned embeddings between the support and query sets because the label spaces between the training and the testing phase are distinct in few-shot learning. 
In particular, previous work e.g., TP~\cite{bennequin2021bridging} investigates the concept of support-query shifts at the density level, denoting a singular shift encompassing a multitude of densities in meta-training and meta-testing phases. Following TP, SQS+~\cite{aimen2022adversarial} further explores the variation in distributions between task-training and task-testing phases, enlarging a distribution mismatch occurring between the support and query sets.
However, in real-world scenarios the shifts are usually unknown and varied, making it difficult to estimate in advance.  Therefore, in this paper, we first propose, RSQS, a realistic but more difficult support-query shift few-shot learning challenge. The key feature of RSQS is that the individual samples in a meta-task are subjected to \textit{multiple} distribution shifts in each support and query set, which is in contrast to previous settings such as TP~\cite{bennequin2021bridging} and SQS+~\cite{aimen2022adversarial} with a single distribution shift across the collective samples within a meta-task.

Concretely, we analyze the RSQS from two aspects, i.e., inter-domain bias and intra-domain variance. On the one hand, inter-domain bias between support and query images may come from the environment where the picture is taken, e.g., foggy and high-luminance, which can significantly damage the model performance, especially in a few-shot application. To alleviate inter-domain bias, optimal transportation~\cite{courty2016optimal} has been regarded as an effective tool to align the learned embeddings over distinct spaces into the same latent space.  On the other hand, the images might be captured by various devices, e.g., smartphones and single-lens reflex cameras, leading to an implicit \textit{intra-domain variance}, between the support and query domains. We theoretically prove that such inter-domain variance could easily mislead the result of optimal transportation. To address intra-domain variance, data augmentation techniques have been used to create and train on more instances to derive a more robust classification model~\cite{zhao2020maximum,gong2021maxup} This is achieved by modifying a single image~\cite{simonyan2014very} or combining multiple images~\cite{yun2019cutmix,zhang2017mixup} at a pixel level as additional training examples. Besides, adversarial training techniques such as projected gradient descent (PGD)~\cite{samangouei2018defense} and AugGAN~\cite{huang2018auggan} are used to find the perturbed images that can confuse the model by predicting an incorrect label as additional training samples. However, these methods usually require numerous iterations to generate the adversarial examples by optimizing a predefined adversarial loss, which is computationally intensive. Additionally, a trade-off has been shown between the models' accuracy and robustness against adversarial examples~\cite{gong2021maxup}.

According to the aforementioned two aspects, in this paper, we proposed \textit{\textbf{DU}al adversarial \textbf{AL}ignment framework (\workname)} to relieve the negative effects from the \emph{inter-domain bias} and \emph{intra-domain variance} in the support-query shift. First, to address the inter-domain bias of unknown shifts in support and query sets, DuaL trains a \textit{repairer network} to correct the query data in an adversarial manner. We identify various shifts in advance to corrupt the original data as the synthesized perturbed inputs to train the repairer network by minimizing the feature-level distance between the original data and the fixed data. Our repairer network is able to solve multiple types of biases with a single model, which demonstrates our framework is a biases-agnostic technique for real-world applications. Second, to reduce intra-domain variance in the support and query sets, DuaL uses a \textit{generator network} to adversarially train a more robust feature extractor by generating perturbed data as the \textit{hard} examples, i.e., less similar to the original data point in the embedding space but should be classified into the same class. 
 DuaL can find the perturbed examples more efficiently than previous adversarial training methods~\cite{aimen2022adversarial} by generating the least similar examples directly, which is the hardest example to enhance the model's generality. Lastly, we use \textit{smooth optimal transport}, which regularizes the negative entropy of the transportation plan to take more query data points as the anchor nodes, leading to a higher error tolerance of the transportation plan, which considers more data points to avoid overfitting in certain samples.  The contributions are summarized as follows.
 
\begin{itemize}
     \item We propose RSQS, a novel realistic support-query shift few-shot learning setting. Compared to the previous challenges, RSQS considers the distribution shift from individual samples in a meta-task subjected to multiple distribution shifts.
    \item We propose \emph{\textbf{Du}al Adversarial \textbf{Al}ignment Framework (DuaL)} to solve RSQS by relieving the inter-domain bias and intra-domain variance. \workname~use corrupted data to train a repairer to reduce the distribution bias, and next, train a robust feature extractor in an adversarial manner to relieve the variance across different domains. To avoid overfitting, regularized optimal transport is used, which adopted the negative entropy of the transportation plan to take more query data points as the anchor nodes. 
    \item We also build a benchmark with seven state-of-the-art baselines of RSQS among three datasets. Extensive experiments manifest that \workname~outperforms other state-of-the-art methods. Lastly, we also derived a theoretical bound and time complexity of \workname~to enhance the completeness of this work. 
\end{itemize}

\begin{figure}[t]
  \centering
  \includegraphics[width = \linewidth]{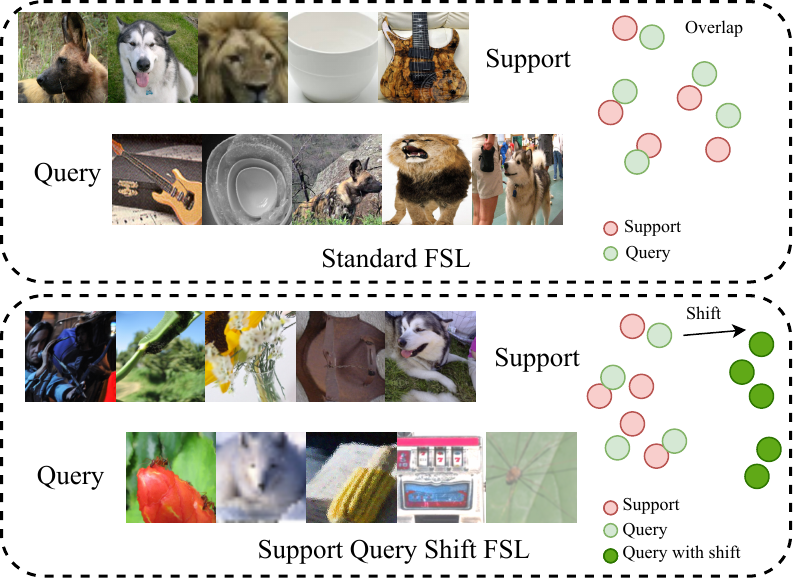}
  \caption{ Illustration of the support-query shift in a 5-way 1-shot classification task, where the support set and the query set are embedded into different distributions.}
  \label{fig:flow_1}
\end{figure}

\begin{figure*}[!t]
  \centering
  \includegraphics[width = \linewidth]{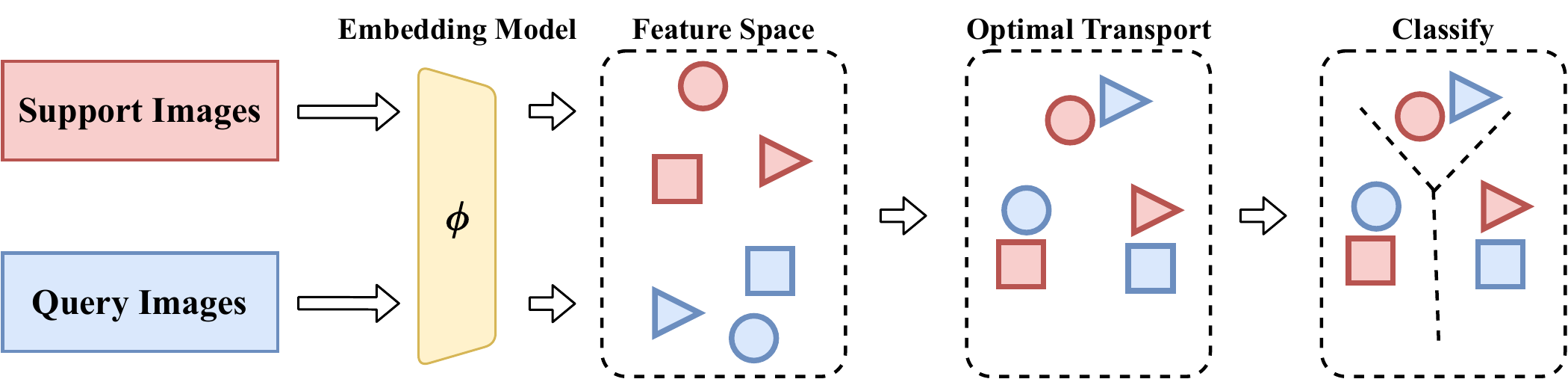}
  \caption{First, we embed the support set and the query set via an embedding model $\phi$ into a feature space. Next, optimal transportation is employed to align the support set (red) and the query set (blue). But, the small perturbations may misguide the transported results, leading to wrong predictions, i.e., classifying the blue circle to the red triangle.}
  \label{fig:flow_2}
\end{figure*}

\section{ Background}

\subsection{Few Shot Learning and Distribution Shifts }
Given a labeled support set
$\mathcal{S} = \cup_{c \in \mathcal{C}} \mathcal{S}^c$, with $\mathcal{C}$ classes, where each class $c$ has $|S^c|$ labeled examples, the goal of few-shot learning is to classify the query set $\mathcal{Q}= \cup_{c \in \mathcal{C}} \mathcal{Q}^c$ into these $\mathcal{C}$ classes. Let $\phi$ denotes the embedding model $\phi(x) \in \mathrm{R}^d $, which encodes the  data point $x$ to the $d$-dimensional feature. $\phi$ is learned from a labeled training set $\mathcal{D} = \{x_i,y_i\}_{i\in[1,|\mathcal{D}|]}$, where $x_i$ is the data point and $y_i$ is the corresponding label. The embedding model can be learned by empirical risk minimization (ERM),
\begin{equation*}
    \min_{\phi,\theta} E_{\{x,y\}\sim \mathcal{D}}[L(\theta(\phi(x)),  y)],
\end{equation*}
where $\theta$ is a trainable parameter to map the embedding $\phi(x_i)$ to the class $y_i$.

Through the embedding model $\phi$, we can encode the data points in support set  (i.e., $x_{s,i} \in \mathcal{S}$) and query set (i.e., $x_{q,j} \in \mathcal{Q}$) to the feature $\phi(x_{s,i})$ and $\phi(x_{q,j})$, respectively. These features are used as input to a comparison function $M$, which measures the distance , e.g., $l_2$-norm, between two samples. Specifically, we classify the query example $\phi(x_{q,j})$ by averaging the embedding $\phi(x^c_{s,i})$ of the support set in class $\mathcal{S}^c$, which can be written as follows.
\begin{align}\small
  \phi^c(x_{s}) =  \frac{1}{|S^c|} \sum_{x_{s,i} \in S^c} \phi(x_{s,i}),  \nonumber
  \\
  y_q =  \argmin_{c} M(\phi^c(x_{s}) ,\phi(x_{q,j})). 
  \nonumber
\end{align}

While the conventional few-shot learning methods assume the support set and the query set lie in the same distribution, a more realistic setting is that the support set $\mathcal{S}$ and the query set $\mathcal{Q}$ follow different distributions, i.e., the support-query shift~\cite{bennequin2021bridging}. While these two sets are sampled from different distributions $\mu_s$ and $\mu_q$, the embeddings for the support set  $\mathcal{S}$ (i.e., $\phi(x_{s})$) and the query set $\mathcal{Q}$ (i.e., $\phi(x_{q})$) are likely to lie in different embedding spaces. Thus, it would lead to a wrong classification result via the comparison module $M$~\cite{bennequin2021bridging}.

Conventional few-shot learning methods assume that the support set and the query set are drawn from the same distribution, but in reality, the support set $\mathcal{S}$ and the query set $\mathcal{Q}$ may come from different distributions, referred to as support-query shift~\cite{bennequin2021bridging}. When the support set $\mathcal{S}$ and query set $\mathcal{Q}$ are sampled from different distributions $\mu_s$ and $\mu_q$, the embeddings for the support set $\phi(x_{s})$ and the query set $\phi(x_{q})$ are likely to be in different embedding spaces. This leads to incorrect classification results when using the comparison module $M(\phi(x_{s}), \phi(x_{q}))$~\cite{jiang2022pgada}.

\subsection{Optimal Transportation}

To align different distributions in support-query shift, optimal transportation~\cite{courty2016optimal} is one of the effective techniques using a transportation plan $\pi(\mu_s, \mu_q)$, which can formally be written as follows.
\begin{align}\small
    W (\mu_s,\mu_q)=  \inf_{\pi \in \Pi(\mu_s,\mu_q)} \int  w(x_s,x_q)d\pi(x_s,x_q),
\end{align}
where $\Pi(\mu_s, \mu_q)$ is the set of transportation plans (or couplings) and $w$ is the cost function, and $W$ is the overall cost of transporting distribution $\mu_s$ to $\mu_q$. In our practice, we select $l_2$-norm of the embedding vector, i.e., $\Vert \phi(x_{s}) - \phi(x_{q})\Vert^2_2$, as our distance function $w$.

Since there are only finite samples for both the support set $x_{s,i} \in \mathcal{S}$ and the query set $x_{q,j} \in \mathcal{Q}$, the discrete optimal transportation adopts the empirical distributions to estimate the probability mass function $\hat{\mu}_s =\sum \delta_{s,i}$, and $\hat{\mu}_q = \sum \delta_{q,j}$, where $\delta_{s,i}$ and $\delta_{q,j}$ is the Dirac distribution. We obtain
\begin{align}\small
\label{eq:basic_ot}
     & \pi^{\ast} = \argmin_{\pi} \sum_{\substack{x_{s,i} \sim \hat{\mu}_{s}\\ x_{q,j} \sim \hat{\mu}_{q}} } w(x_{s,i},x_{q,j}) \pi(x_{s,i},x_{q,j})
\end{align}
Then, Sinkhorn's algorithm~\cite{cuturi2013sinkhorn} is adopted to solve the optimal transportation plan $\pi^{\ast}$.

Equipped with the optimal plan $\pi^{\ast}$, we transport the embeddings of the support set $\phi(x_{s,i})$ to $\hat{z}_{s,i}$ by barycenter mapping~\cite{courty2016optimal} to adapt the support set to the query set.
\begin{equation}\small
    \hat{\phi}(x_{s,i}) =\frac{\sum_{x_{q,j}\in\mathcal{Q}}\pi^{\ast}(x_{s,i},x_{q,j}) \phi(x_{q,j}) }{ \sum_{x_{q,j}\in\mathcal{Q}}\pi^{\ast}(x_{s,i},x_{q,j})}.
    \label{eq:trans_q}
\end{equation}
$\hat{\phi}(x_{s,i})$ denotes the transported embedding of $x_{s,i}$. Therefore, we can correctly measure the distance metric  $M(\hat{\phi}(x_{s,i}), \phi(x_{q,j}))$ in a shared embedding space.

\subsection{The Harm of Perturbation in Optimal Transport }
\label{sec:HarmOfShift}

While optimal transport has achieved convincing results in support-query shift few-shot learning, there is another challenge that comes from the quality of the embedding $\phi(x)$, i.e., intra-domain variance, in the embedding may mislead the transportation plan. For example, clean images' embeddings may give a better plan than foggy images since fog brings additional noise into the original data. Formally, such misestimation can be defined as follows.

\begin{lemma} The error of the transportation cost is 
    \begin{equation*}\small
      W_\sigma(\mu_s,\mu_q)  \leq W(\mu_s,\mu_q) \leq  W_\sigma(\mu_s,\mu_q) + \sqrt{d(\sigma_s^2 + \sigma_q^2)},
    \end{equation*} 
    
    where $W_\sigma(\mu_s,\mu_q) \coloneqq W(\mu_s*\mathcal{N}_{\sigma_s},\mu_q*\mathcal{N}_{\sigma_q})$ denotes the original support and query set distribution $\mu_s$ and $\mu_q$ being perturbed with Gaussian noises $\sigma_s$ and $\sigma_q$. 
\label{lemma:err}
\end{lemma}
Note that  $|\cdot|$ is the absolute value, and $*$ is the convolution operator. Based on Lemma~\ref{lemma:err}, we estimate the error of transported embedding $\hat{\phi}(x_{s,i})$ in Eq. (\ref{eq:trans_q}).

\begin{proof}
The left side inequality immediately follows because $W$
is non-increasing under convolutions, since $\mathcal{N}_{\sqrt{\sigma_s^2+\sigma_q^2}}$ = $\mathcal{N}_{\sigma_s}*\mathcal{N}_{\sigma_q}$, where $*$ is the convolution operator. 

In the right side of the inequality, we adopt Kantorovich-Rubinstein duality to write the optimal transport as follows.
\begin{align}\small
    W(\mu_s,\mu_q) = \sup_{\Vert w \Vert_{Lip} \leq 1 } E_{\mu_s} [w] - \mathrm{E}_{\mu_q} [w]
\end{align}

\begin{align}\small
    W_\sigma(\mu_s,\mu_q) = \sup_{\Vert{w} \Vert_{Lip} \leq 1 } E_{\mu_s*\mathcal{N}_{\sigma_s}} [w_{\sigma}] - \mathrm{E}_{\mu_q*\mathcal{N}_{\sigma_q}} [w_{\sigma}]
\end{align}
where $\Vert{\cdot}\Vert_{Lip}$ is the Lipschitz norm. Letting  $w^{\ast}$ be optimal for $W(\mu_s,\mu_q)$, we obtain,
\begin{align}\small
     W_\sigma(\mu_s,\mu_q) =  E_{\mu_s*\mathcal{N}_{\sigma_s}} [w^{\ast}] - \mathrm{E}_{\mu_q*\mathcal{N}_{\sigma_q}} [w^{\ast}].
    \label{eq:exp}
\end{align}
Let $X_s\sim \mu_s$, $Z_s \sim {N_{\sigma_s}}$ as independent random variables, we have,
\begin{align}\small
    \label{eq:z_s}
     &| E_{\mu_s}[w^{\ast}] - E_{\mu_s*N_{\sigma_s}}[w^{\ast}] |\\  \nonumber
     = & E [w^{\ast} (X_s)] - E [w^{\ast} (X+Z_s)] \\
     \leq& E [\Vert{Z_s}\Vert_2^2] =  \sqrt{d} \sigma_s. \nonumber
\end{align}
where the last inequality uses $\Vert{w^\ast}\Vert_{Lip} \leq 1$. $d$ is the dimension of the embedding vector. Similarly, $X_q\sim \mu_q$, $Z_q \sim {N_{\sigma_q}}$ as independent random variables, we have,
\begin{align}\small
    \label{eq:z_q}
     &| E_{\mu_q}[w^{\ast}] - E_{\mu_q*N_{\sigma_q}}[w^{\ast}] |\\  \nonumber \nonumber
     = & E [w^{\ast} (X_q)] - E [w^{\ast} (X+Z_q)] \\
    \leq & E [\Vert{Z_q}\Vert_2^2] =\sqrt{d} \sigma_q. \nonumber
\end{align}
By inserting Eq.~(\ref{eq:z_s}) and Eq.~(\ref{eq:z_q}) to Eq.~(\ref{eq:exp}), and Cauchy-Schwarz inequality,  we concludes the proof.
\end{proof}


\begin{theorem} The error of the transported embedding is 
    \begin{align}\small
        E[\Vert{\hat{\phi}(x_{s,i}) - \hat{\phi}_{\sigma}(x_{s,i})}\Vert_2^2] = \sqrt{ d(\sigma_s^2 +\sigma_q^2)}, 
    \end{align} 
where $\hat{\phi}_{\sigma}(x_{s,i})$ is the transported embedding from the perturbed distribution $W_\sigma(\mu_s,\mu_q)$. 
    \label{thm:err}
\end{theorem}

\begin{proof}
Base on Lemma~\ref{lemma:err}, barycentric coordinate is defined as follows,
\begin{equation}\small
    \hat{\pi}^{\ast}_i = \frac{\pi^{\ast}(x_{s,i},x_{q,j}) }{ \sum_{x_{q,j}\in\mathcal{Q}}\pi^{\ast}(x_{s,i},x_{q,j})} \sim \mathcal{N}_{\sigma_s} 
    \label{eq:xi}
\end{equation}
Let $X_q\sim \mu_q$, $X^\sigma_q \sim {\mu_q * N_{\sigma_q}}$ as independent random variables, 
\begin{equation}\small
    E[X^{\sigma(t)}_q - X_q^{(t)}] = \sigma_q,
    \label{eq:err_q}
\end{equation}
where $X^{\sigma(t)}_q$ and $X_q^{(t)}$ denotes the $t$-th dimension of random variable $X^\sigma_q$ and $X_q$, respectively.

Combining Eq. (\ref{eq:xi}) and Eq. (\ref{eq:err_q}), the projected distribution $\hat{X}_s  \sim {\mu_s * N_{\sigma_s} * N_{\sigma_q}} = {\mu_s * N_{\sqrt{\sigma_s^2+\sigma_q^2}}}$.
\begin{equation}\small
    E[\hat{X}_s - X_s] = \sqrt{d(\sigma_s^2 + \sigma_q^2)}.
\end{equation}

As the noise level, i.e., $\sigma_s$, and $\sigma_q$, increases, it is more likely to mislead the transportation plan and alleviate the model's performance. Therefore, it's non-trivial to learn a better embedding function $\phi(x)$ having a better capability of noise tolerance such that $\phi(x_p) \approx \phi(x)$, where $x_p$ is original image $x$ with the small perturbation. 

\end{proof}

\section{Dual Adversarial Alignment Framework (\workname)}
\label{sec:method}
As we mentioned in Sec.~\ref{sec:HarmOfShift}, the low embedding quality will mislead the transportation plan. Therefore, we aim to obtain a \emph{clean} feature embedding to relieve such misleading.
In this section, we present our framework, \workname, from two perspectives. We first propose semantic-aware data generation and self-supervised manner adversarial training to obtain clean features for optimal transportation. Next, a regularized optimal transportation is introduced to align support and query set better. The workflow of \workname~is presented in Fig.~\ref{fig:flow_2}.

\subsection{Realistic Support-Query Shift Few-shot
Learning (RSQS)}
\label{Sec:RSQS}

In support-query shift few-shot learning setup~\cite{bennequin2021bridging,aimen2022adversarial},  a disjoint distribution exists even at the task composition level, i.e., $\mathcal{T}_{S} \neq \mathcal{T}_{Q}$, where $\mathcal{T}_{S}$,   $\mathcal{T}_{Q}$ are the distributions on support and query sets respectively. 
Let $\mathcal{\mathbf{D}}_{Train}$ and $\mathcal{\mathbf{D}}_{Test}$ be the set of domains for the training and testing phases.  RSQS follows the standard SQS assumption of disjoint meta-train and meta-test domains, $\mathcal{\mathbf{D}}_{train} \cap \mathcal{\mathbf{D}}_{test} = \varnothing$.  The distribution of the support set and the query set of RSQS setting is mismatched by \textit{multiple} unknown shifts, i.e.,  $|\mathcal{T}_{S}| \neq  |\mathcal{T}_{Q}|$, where $|\cdot|$ denotes the number of different shifts, shown in Fig.~\ref{fig:multi_vis}. In addition, the shifts are also different in each sample among the support set and query set in the meta-testing phase of RSQS, i.e., for each instance $q_i, q_j \in \mathcal{Q}$, $\mathcal{T}_Q^i \neq \mathcal{T}_Q^j$, same to the support set.

As shown in Table~\ref{table:different}, conventional machine learning (ML) and FSL differs from the distribution shift between the training and testing phase. TP~\cite{bennequin2021bridging} define the support-query shift few-shot learning but assumed only a similar but disjoint support-query shift in the meta-train and meta-test datasets. SQS+~\cite{aimen2022adversarial} considers an unknown shift that occurs in the support and query set. However, RSQS enhances the uncertainties stemming from unknown shifts in the support and query sets by implementing multiple shifts within each instance of a meta-task.

\begin{table}[!ht]
\resizebox{\columnwidth}{!}{
    \centering
    \begin{tabular}{c|c|c|c|c}
    \hline
        Type  & 
        \makecell[c]{Training and \\ Testing Set} & 
         \makecell[c]{Support Query \\ Shifts} & 
        \makecell[c]{Multiple \\ Shifts}  & 
        \makecell[c]{Instance \\ Shifts}  \\ \hline 
    ML & \XSolidBrush  & \XSolidBrush   & \XSolidBrush & \XSolidBrush \\
    FSL~\cite{snell2017prototypical} & \CheckmarkBold & \XSolidBrush   & \XSolidBrush & \XSolidBrush \\
    TP~\cite{bennequin2021bridging} &  \CheckmarkBold & \CheckmarkBold   & \XSolidBrush & \XSolidBrush \\
    SQS$+$~\cite{aimen2022adversarial} & \CheckmarkBold & \CheckmarkBold   & \XSolidBrush & \XSolidBrush  \\
    \makecell[c]{RSQS (Ours)}  & \CheckmarkBold & \CheckmarkBold   & \CheckmarkBold & \CheckmarkBold \\ \hline
    \end{tabular}
}
    \caption{Comparison of different shift settings.}
    \label{table:different}
\end{table}

\begin{figure}[t]
  \centering
  \includegraphics[width = \linewidth]{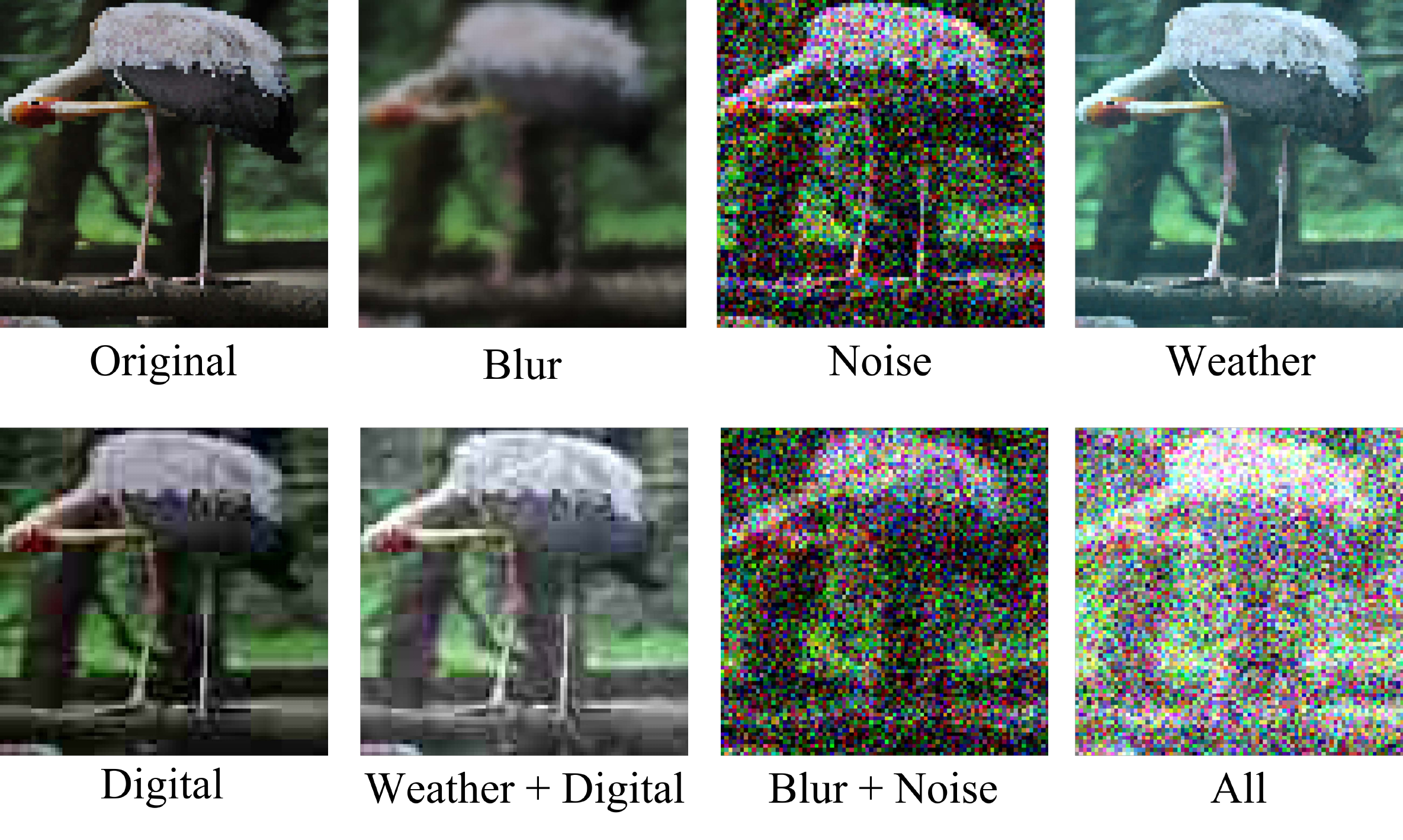}
  \caption{ Visualization of the impact on multiple shifts. In contrast to the conventional SQS few-shot learning, in RSQS, multiple shifts are adopted for each meta-testing task which is more closely to the complex and difficult real-world scenarios.}
  \label{fig:multi_vis}
\end{figure}


\begin{figure*}[t]\small
  \centering
  \includegraphics[width=1.02\linewidth]{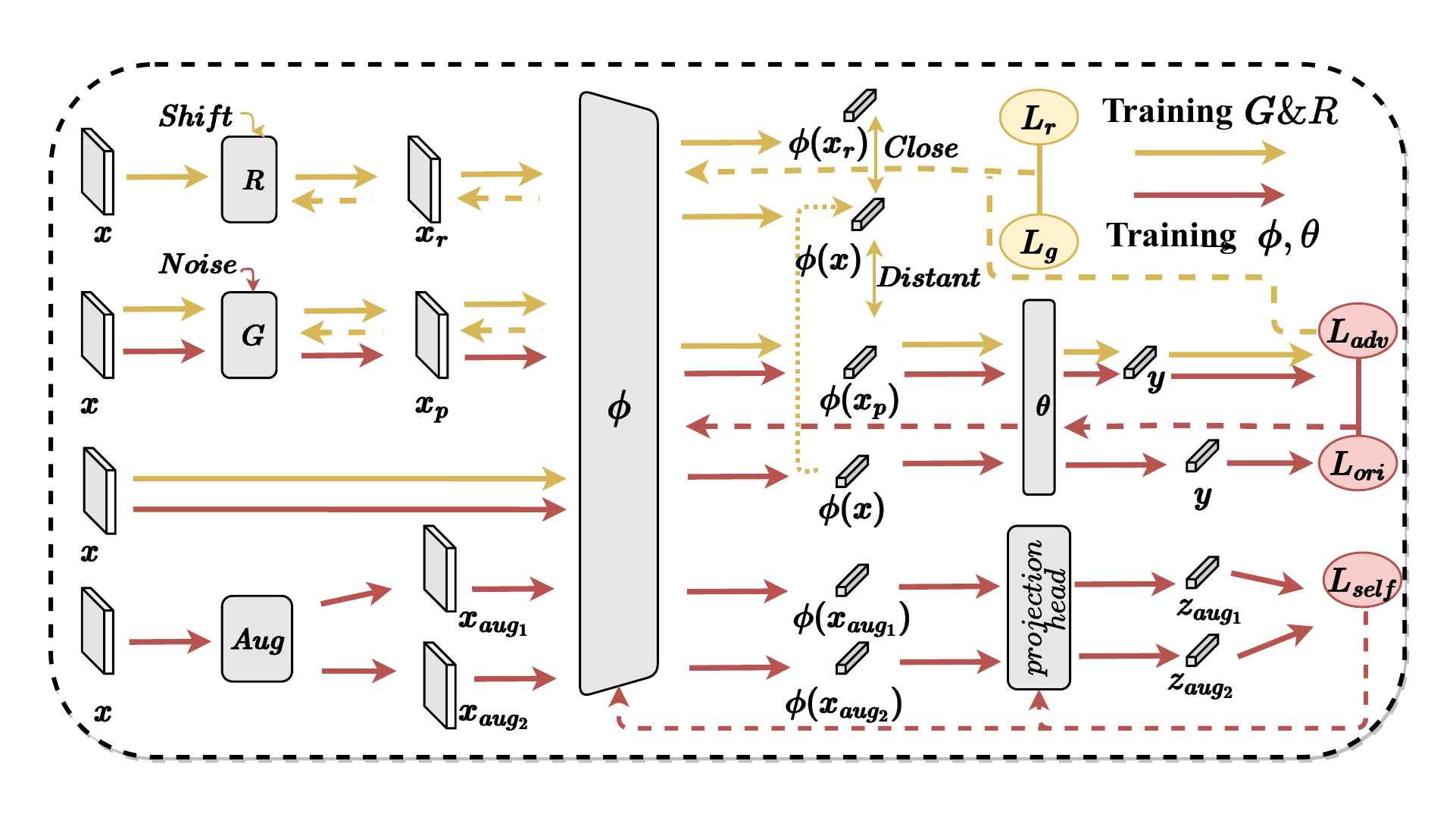}
  \vspace{-20pt}
  \caption{Architecture of \workname. The yellow and red line denotes the training $G \& R$ and $\phi \& \theta$ phase, respectively. The solid lines represent forward propagation, and dashed lines denote backward propagation of gradients. }
  \label{fig:representation_b}
\end{figure*}
\subsection{Semantic-Aware Data Generation} 
In this work, to alleviate inter-domain bias and intra-domain variance in images and derive a more robust embedding function, our target is to generate a set of augmented data since optimal transport may be misguided with low-quality embeddings and lead to incorrect predictions. 

First, to relieve the inter-domain bias from kinds of shifts in real applications, we  generate a perturbed data $x_r$ which is \textit{more} similar to the original data $x$. The  objective function is as follows. 
\begin{equation}\small
    \min_{\phi} \mathbb{E}_{\{x\}\sim \mathcal{D}}[\min_{x_{r}} M(\phi(x_{r}),  \phi(x))],
    \label{eq:adversarial_more}
\end{equation}
 where $M$ is the comparison module in few-shot learning, e.g., $l_2$-norm.  We force $x_r$ to be close to $x$ since we hope the repairer $R$ can learn the particular shift, which needs to produce a clean image to relieve the shifting degree in the inference phase. Hence, to efficiently obtain $x_r$, we use a dual semantic-aware shifting repairing generator to produce the $x_r$, i.e.,\footnote{It is worthy to note that $x_p$ in Eq.~(\ref{eq:adversarial}) is utilized as a challenging example to increase the reliability of our feature extractor.}
 
\begin{equation}\small
    x_r = R(\hat{x}), 
\end{equation}
where $\hat{x}$ is the original data with a particular shift. $R$ is a learnable model to relieve the shift with 3-layer convolutional neural network in our
practice.

Second, to relieve the  intra-domain variance from kinds of perturbations, we follow the framework~\cite{gong2021maxup} which generates augmented data by \textit{minimizing the maximum loss} over the augmented data $x_{p}$, which can be formally written as follows,
\begin{equation}\small
    \min_{\phi,\theta}  \mathbb{E}_{\{x,y\}\sim \mathcal{D}}[\max_{x_{p}} L(\theta(\phi(x_{p})),  y)].
    \label{eq:maxup}
\end{equation}
Note that Eq.~(\ref{eq:maxup}) can be easily minimized with stochastic gradient descent (SGD). 
In practice, we can sample a batch of augmented data $x_{p}$ and compute the gradient of the data point $x_{p}^{\ast}$ with the highest loss $L$. Therefore, the model would learn the hardest example overall augmented data $x_{p}$.

However, it's hard to collect sufficient labels for each class in a few-shot learning. Therefore, instead of maximizing the empirical risk of the labeled data by Eq. (\ref{eq:maxup}), we introduce a self-supervised learning based objective.
\begin{equation}\small
    \min_{\phi}  \mathbb{E}_{\{x\}\sim \mathcal{D}}[\max_{x_{p}} M(\phi(x_{p}),  \phi(x))],
    \label{eq:adversarial}
\end{equation}
Through maximizing the distance between  $\phi(x_{p})$ and  $\phi(x)$, we are able to generate perturbed the image $x_p$ that is \emph{less} similarly to the original image $x$ in embedding space as hard example in a few-shot learning setting. To effectively generate the perturbed data (shown in Fig.~\ref{fig:representation}), we introduce a semantic-aware perturbation generator to synthesize the augmented data.
\begin{equation}\small
    x_p = G(x) \text{ , s.t. } \Vert{x_p - x}\Vert_2^2 \leq \epsilon
\end{equation}

where $G$ is a model to generate the perturbed image $x_p$ with a same architecture of $R$. Besides, we utilize dropout~\cite{ian2014generative} to provide the necessary randomness of our model. Compared to conventional adversarial training technique~\cite{samangouei2018defense,wang2018low}, which usually sample the perturbed images from an i.i.d distribution $x_{p} \sim P(\cdot|x)$, e.g., Gaussian distribution, our method can encode the semantic of the image $x$ without requiring many samples to achieve convergence~\cite{gong2021maxup}.

\subsection{Self-Supervised Manner Adversarial Training }
At training phase (shown in Fig.~\ref{fig:representation_b}),  we also minimize the empirical risk of the perturbed data $x_p$ to ensure that the generator is able to persist enough information to predict the same class $y$ with KL divergence~\cite{phoo2020self}. The overall objective is,
\begin{equation}\small
   \max_G  M(\phi(G(x)),  \phi(x)) - KL(\theta(\phi(G(x))), y)). 
\label{eq:overall_1}
\end{equation}
Then, we adopt stochastic gradient descent (SGD)~\cite{bottou2012stochastic} to train our generator $G$. It is worth noting that we fix the parameters of the embedding function $\phi$ and $\theta$ when training the generator $G$ to stabilize the training process~\cite{antoniou2017data}.

To effectively train the embedding function, we minimize the empirical risk of the original data $x$ and the perturbed data point $x_p$ (i.e., hard example) via KL divergence.
\begin{equation*}\small
    L_{ori} =  KL(\theta(\phi(x)), y)),
\quad
    L_{adv} = KL(\theta(\phi(x_p)), y)) .
\end{equation*}
We also leverage the unlabeled data to enhance the generalbility of the emebddings by the auxiliary contrastive self-supervised learning~\cite{chen2020simple}. 

Summing up, the overall objective is,
\begin{equation}\small
  \min_{\phi,\theta} L_{ori} + \lambda_1 L_{adv} + \lambda_2 L_{self},
\label{eq:overall}   
\end{equation}
where $\lambda_1$ and $\lambda_2$ is the trade-off parameter between each loss; $L_{self}$ is the NT-Xent Loss~\cite{chen2020simple}. Similarly, the classifier $\theta$ is trained by minimizing $L_{ori}$ and $L_{adv}$. Detailed pseudo code is illustrated in Algorithm \ref{alg:prada}.

\subsection{Optimal Transport with Weighted Negative Entropy Regularization}
For further align the support query set better, a general optimal transport plan is necessary~\cite{cuturi2013sinkhorn}. We introduce a weighted negative entropy regularization which is described as follows:
\begin{align}\small
\label{eq:smooth_ot}
      \pi^{\ast} =   \argmin_{\pi} & \beta\sum_{\pi \in \Pi(\hat{\mu}_{s},\hat{\mu}_{q})} w(x_{s,i},x_{q,j}) \pi(x_{s,i},x_{q,j}) \notag \\ 
      & +(1-\beta)\sum_{i,j} \pi(i,j) \log \pi(i,j),
\end{align}
where $\beta$ is the weight parameter. Compared to Eq.~(\ref{eq:basic_ot}), the second term in Eq.~(\ref{eq:smooth_ot}) is the negative entropy regularization. The intuition of this regularization term is that we hope the transport plan should not be zero with a high probability. For example, if the transport plan is sparse, it will obtain a heavy penalty from Eq.~(\ref{eq:smooth_ot}), which can be regarded as a smooth version compared with classical optimal transport.  As a result, the optimal transport plan will have a more general and denser coupling between the distributions, leading to a better alignment in the support-query shift.

\begin{algorithm}[!ht]
    \caption{DuaL}
    \begin{algorithmic}[1]
    \REQUIRE  Dataset $\mathcal{D}$, comparison module $M$, learning rate $\eta$, trade-off parameters $\lambda_1$ and $\lambda_2$, arbitrary Shifting $S$.
    \ENSURE Embedding model $\phi$, repairer $R$, Optimal Plan $\pi^{*}$,
    \STATE Initialize generator $G$, repairer $R$, 
    \STATE Initialize embedding model $\phi$, classifier $\theta$.
    \FOR{$\{x,y\}$ in $\mathcal{D}$}
        
        \STATE \textit{\# fixed $\phi, \theta$, update $G$, and $R$} 
        \STATE $x_p$ = $G(x)$,
        \STATE $x_r$ = $R(S(x))$,
        \STATE $L_{g}$ = $-M(\phi(x_p),  \phi(x))$,\ $L_{r}$ = $M(\phi(x_r),  \phi(x))$, 
        \STATE $L_{adv}$ = $KL(\theta(\phi(x_p)),y))$
        \STATE \textit{\# Generated less similar data points with perturbations.}
        \STATE {$G \gets G - \eta \nabla(L_{g} + L_{adv})$ }, \ {$R \gets R - \eta \nabla L_{r}$ } 
        
        \STATE \textit{\# fixed $G$, $R$, update $\phi, \theta$}
        \STATE $x_p$ = $G(x)$,
        \STATE $L_{ori}$ = $KL(\theta(\phi(x)), y))$, 
        \ $L_{adv}$ = $KL(\theta(\phi(x_p)),y))$
        \STATE \textit{\# classifying the generated samples correctly.}    
        \STATE $\phi \gets \phi - \eta \nabla$ ($L_{ori} + \lambda_1 L_{adv} + \lambda_2 L_{self})$ 
        \STATE $\theta \gets \theta - \eta \nabla$ ($L_{ori} + \lambda_1 L_{adv} )$
    \ENDFOR
\STATE Solve the Eq.~\ref{eq:smooth_ot} to obtain $\pi^{*}$, and transport support sets.      
\end{algorithmic}
\label{alg:prada}
\end{algorithm}
\section{Analysis}

\subsection{Upper Bound of \workname}
Here, we provide the upper bound of \workname, which can be regarded as the formal presentation of the harm of shifts.

\begin{theorem}
\label{lemma:erm_variance} 
Let $\phi(x; \theta)$ be $\lambda$-strongly convex with respect to $\theta$ for all $x$, define the empirical risk minimizer under the augmented data via
\begin{equation}
\hat{\theta}=\argmin_{x\in\mathbb{R}^{d}}\frac{1}{n}\sum_{t=1}^{n}\phi(x; \theta).
\end{equation}

Then the empirical risk minimizer enjoys the guarantee
\begin{equation}
\label{eq:erm_parameter}
\mathbb{E} \lVert {\hat{\theta}-\theta^*} \rVert^{2}_2 \leq{} \frac{4\sigma_g^{2}}{\lambda^{2}n}.
\end{equation}
\end{theorem}
where $\sigma_g $ is  the variance of the gradient of the augmented loss function (i.e., the trace of the covariance matrix $\sigma_g = \mathbf{Var}_X[\nabla_\theta \phi(\theta, X) ] $), 
$\theta^{*}$ is the minimizer of the augmented population risk $\mathbb{E}\phi(\theta,X)$, 
$\hat{\theta}$ is the minimizer of the augmented empirical risk $\mathbb{E}\phi(\theta,X)$.
\begin{proof}
Let ${F}_n(\theta)=\frac{1}{n}\sum_{i=1}^{n}\phi(\theta;x_i)$ be the empirical objective. Since $\phi(\theta;x_i)$ is $\lambda$-strongly convex for each $x_i$, ${F}_n$ is itself $\lambda$-strongly convex, and so we have
\begin{equation}
\langle { \nabla \hat{F}_n(\theta^{\star}),\hat{\theta}-\theta^*} \rangle + \frac{\lambda}{2}\  \lVert \hat{\theta}-\theta^* \rVert^{2}_2 \leq{} \hat{F}_n(\hat{\theta}) - \hat{F}_n(\theta^{\star}).  \notag
\end{equation}
Since, $\hat{\theta}$ is the empirical risk minimizer, we have $\hat{F}(\hat{\theta})-\hat{F}(\theta^*)\leq 0 $. Therefore,
\begin{equation}\small
\frac{\lambda}{2} \lVert {\hat{\theta}- \theta^*} \rVert^{2}_2
\leq{} \langle \nabla \hat{F}_n(\theta^*), \hat{\theta} - \theta^* \rangle 
\leq{} \lVert \nabla \hat{F}_n(\theta^*) \rVert \lVert \hat{\theta}-\theta^* \rVert.  \notag
\end{equation} 
If $\hat{\theta} - \theta^* = 0$, the proof is done. Otherwise, we notice that,
\begin{equation}
\lVert \hat{\theta}-\theta^* \rVert  
\leq{} \frac{2}{\lambda}\lVert {\nabla \hat{F}_n(\theta^*)} \rVert.      \notag
\end{equation}
Since $\lVert \hat{\theta}-\theta^* \rVert \leq \frac{\sigma_g^2}{n}$, and both sides are squared and then take the expectation, which gives
\begin{equation}
\lVert \hat{\theta}-\theta^* \rVert _{2}^{2}
\leq{} \frac{4\sigma_g^2}{\lambda^2 n}. \notag
\end{equation}
\end{proof}

\subsection{Time Complexity}

Here, we consider the time complexity of  \workname in the training phase since most computation exists in this stage.

For a general convolutional neural network (CNN),  the total time complexity of all convolutional layers is:
\begin{equation}\label{eq:time}
T_{CNN} = 
O\left(\sum_{l=1}^{L}n_{l-1} \cdot   n_{l} \cdot s_{l}^2 \cdot m_{l}^2\right),
\end{equation}
where $L$ is the depth of convolutional layers. $n_{l}$ is the number of filters in the $l$-th layer and $n_{l-1}$ is also known as the number of input channels of the $l$-th layer. $s_{l}$ is the spatial size of the filters and $m_{l}$ is the spatial size of the output feature map. Note that the time complexity of forward and backward propagation varies on a constant level~\cite {he2015convolutional}.

In addition, for a general fully connected network (FCN), the total time complexity of the forward pass is:

 \begin{equation}\label{eq:time}
 T_{FCN} = 
 O\left(\sum_{l=1}^{L} z_{l-1} \cdot N_{l-1} \cdot N_{l}  \right),
\end{equation}
where $z_{l-1}$ is the size of input of $l$-layer (number of training examples), $N_{l}$ is the the number of neurons in $l$-th layer. 

In the first part of \workname, we fix the $\phi$ and train $G$ and $R$. Thus, in this phase, $\phi$ only has its inference computational overhead, including the time complexity of CNN and FCN.  The time complexity in the first part is: 

\begin{equation}\label{eq:time}
T_{1} = T_{CNN}^{R(I)} + T_{CNN}^{G(T)} + T_{CNN}^{\phi(I)} + T_{FCN}^{\phi(I)},
\end{equation}
where, $ T_{CNN}^{R(I)} $ and $T_{CNN}^{G(T)}$ denotes training phase time complexity of training $R$ and $G$, respectively. $T_{CNN}^{\phi(I)}$ and $T_{FCN}^{\phi(I)}$ denote the inference time complexity of CNN and FCN parts of $\phi$, respectively.

In the second part of \workname, to train $\phi$, we fix $G$ and $R$. Similarly, the time complexity in the second part is: 
\begin{equation}\label{eq:time}
T_{2} =  T_{CNN}^{G(I)} + T_{CNN}^{\phi(T)} + T_{FCN}^{\phi(T)},
\end{equation}
where $T_{CNN}^{G(I)}$ denotes the time complexity of inference generator network $G$. $T_{CNN}^{\phi(T)}$ and $T_{FCN}^{\phi(T)}$ denotes the time complexity of training the CNN and FCN of $\phi$, respectively. Thus, the final time complexity of \workname~is shown as follows,
\begin{equation}\label{eq:time}
T = T_{1} + T_{2}. 
\end{equation}

\section{Evaluation}
We build a benchmark of RSQS and evaluate 8 methods (\workname~and other 7 state-of-the-art baselines) on three public datasets including an experiments setup, main results, and ablation studies.

\subsection{Experiment Setup}
\label{sec:experiment_setup}

\subsubsection{Shifting Setting in RSQS}

Following \cite{hendrycks2019benchmarking}, we classified the  
shifts, also known as perturbations, into four main categories consisting of 15 instances in total: \textit{Noise}, \textit{Blur}, \textit{Weather}, and \textit{Digital}. In the training phase, we specifically utilize \textit{Gaussian Noise}, \textit{Defocus Blur}, \textit{Glass Blur}, \textit{Snow}, \textit{Frost}, \textit{Contrast}, and \textit{Elastic} shifts. The validation phase involves \textit{Shot Noise}, \textit{Motion Blur}, \textit{Fog}, and \textit{Pixel} shifts, while the testing phase incorporates \textit{Impulse Noise}, \textit{Zoom Blur}, \textit{Bright}, and \textit{JPEG} shifts. Note that the shifts are disjoint at the instance level throughout the training, validation, and testing stages. In each meta-task during the meta-testing phase, we randomly apply multiple shifts to each sample in the support and query sets. In this paper, we employ a maximum of 4 different shifts in each meta-task.

\subsubsection{Datasets}
Three standard benchmark datasets (CIFAR100, mini-Imagenet, and tiered-Imagenet) in few-shot learning are employed to validate our framework. 

\begin{itemize}
    \item \textbf{CIFAR100.}~\cite{krizhevsky2009learning} consists of $60,000$ three-channel square images of size $32 \times 32$, evenly distributed in $100$ classes. Classes are evenly distributed in $20$ superclasses.  We employ $19$ image transformations~\cite{zhang2020adaptive}, each one being applied with $5$ different levels of intensity, to evaluate the robustness of a model to support-query shift. The training dataset, validation dataset and testing dataset have $2,184,000$, $114,000$, $330,000$ images, respectively. For fair comparison, we adopt different transformations on each dataset to evaluate the robustness of our model.
    \item \textbf{mini-ImageNet}~\cite{triantafillou2019meta} contains $60,000$ square images with three channels of size $224 \times 224$ from the ImageNet dataset with a $64$-classes training set, a $16$-classes validation set, and a $20$-classes test set~\cite{vinyals2016matching}.  Similar to CIFAR100, mini-ImageNet also has the same transformations proposed by~\cite{hendrycks2019benchmarking} to simulate different domains~\cite{hendrycks2019benchmarking}. The training dataset, validation dataset and testing dataset have $1,200,000$, $182,000$, $228,000$ images, respectively.  
    \item \textbf{Tiered-ImageNet}~\cite{ren2018meta} contains of 779,165 three-channel $84 \times 84$ images, grouped  into 34 higher-level nodes in 608 classes.  The set of nodes is partitioned into 20, 6, and 8 disjoint sets of training, validation, and testing nodes, and the corresponding classes form the respective meta-sets.
\end{itemize}

\subsubsection{Baselines}
\label{sec:baseline}
We mainly have two types of baselines i.e., w/o and w/ adversarial data augmentation works. In the first type of baselines which have $7$ different state-of-the-art few-shot learning methods.

\begin{enumerate}
    \item \textbf{MatchingNet~\cite{vinyals2016matching}.} MatchingNet measures the pairwise cosine similarity between the support set and the query set and assigns the same class of the support example to the query example.
    \item \textbf{{ProtoNet~\cite{snell2017prototypical}.}} Instead of pairwise similarity, ProtoNet uses euclidean distance to classify queries to the prototype embeddings, i.e., averaging the embeddings of all support examples in the same class.
    \item \textbf{{TransPropNet~\cite{liu2018learning}.}} TransPropNet is an extension of ProtoNet, which utilizes graph neural network of label, leveraging information of local neighborhoods. 
    \item \textbf{{FTNET~\cite{dhillon2019baseline}.}} FTNET is a meta-learning framework that estimates the distribution between the training set and testing set transductively.
    \item \textbf{{TP~\cite{bennequin2021bridging}.}} Transported Prototypes (TP) combines the ProtoNet, optimal transport and transductive batch normalization to solve the support-query shift in few-shot learning. 
    \item \textbf{AQP~\cite{aimen2022adversarial}.} AQP aims to create more challenging virtual query sets by adversarially perturbing the query sets, inducing a distribution shift between support and query sets.  Note that AQP can be regarded as the SOTA method in support-query shift few-shot learning using episodic training.
    \item  \textbf{{PGADA~\cite{jiang2022pgada}}.} PGADA alleviates the misestimation of optimal transportation by learning hard examples that are generated by a self-supervised mechanism. Also, a smooth transportation plan is obtained with the help of negative entropy regularization. Note that PGADA can be regarded as the state-of-the-art (SOTA) method in support-query shift few-shot learning using non-episodic training.

\end{enumerate}

\subsubsection{Evaluation and Implementation Details}
\begin{table*}[ht]
\resizebox{2\columnwidth}{!}{
    \centering
    \begin{tabular}{c|c|c|c|c|c|c}
    \hline
        \multirow{2}{*}{Dataset}  & \multicolumn{2}{c|}{CIFAR-100} & \multicolumn{2}{c|}{mini-ImageNet}  & \multicolumn{2}{c}{Tiered-ImageNet}  \\ \cline{2-7} 
        & 1 shots & 5 shots & 1 shots & 5 shots & 1 shots & 5 shots \\ 
        \hline 
        \multicolumn{7}{c}{8 Targets} \\ 
        \hline
        MatchingNet~\cite{vinyals2016matching}. & $30.26_{\pm0.38}$ & $40.35_{\pm0.33}$ &  $43.62_{\pm0.47}$ & $55.03_{\pm0.44}$ & $30.01_{\pm0.41}$ & $39.75_{\pm0.41}$ \\ 
        ProtoNet~\cite{snell2017prototypical} & $28.53_{\pm0.30}$ & $42.13_{\pm0.34}$  & $43.84_{\pm0.44}$ & $59.18_{\pm0.37}$ & $30.15_{\pm0.41}$ & $43.41_{\pm0.43}$ \\ 
        TransPropNet~\cite{liu2018learning} &  $31.01_{\pm0.34}$ &  $38.67_{\pm0.32}$ & $24.22_{\pm0.29}$ & $25.83_{\pm0.25}$ & $24.18_{\pm0.32}$ & $27.48_{\pm0.35}$ \\ 
        FTNET~\cite{dhillon2019baseline} &  $22.36_{\pm0.21}$  & $26.19_{\pm0.25}$ & $37.04_{\pm0.44}$  & $49.07_{\pm0.43}$  & $22.01_{\pm0.30}$ & $24.73_{\pm0.32}$ \\ 
        TP~\cite{bennequin2021bridging}  & $30.89_{\pm0.42}$  &  $43.44_{\pm0.45}$ & $45.66_{\pm0.55}$ & $62.14_{\pm0.47}$ & $29.34_{\pm0.43}$ & $42.57_{\pm0.44}$ \\ 
        AQP~\cite{aimen2022adversarial}.  & $32.53_{\pm0.42}$ & $46.22_{\pm0.39}$  & $46.77_{\pm0.42}$ & $61.83_{\pm0.50}$ & $30.22_{\pm0.42}$ & $43.58_{\pm0.44}$ \\
        PGADA (Proto)~\cite{jiang2022pgada} & $34.90_{\pm0.45}$  & $49.42_{\pm0.46}$ & $50.37_{\pm0.57}$ & $65.25_{\pm0.46}$ & $28.47_{\pm0.40}$ & $36.45_{\pm0.37}$ \\ 
        PGADA (Matching)~\cite{jiang2022pgada} &  $35.16_{\pm0.46}$ & $44.82_{\pm0.43}$ & ${52.23_{\pm0.57}}$  & $61.05_{\pm0.46}$  & ${28.63_{\pm0.41}}$  &  $35.29_{\pm0.37}$    \\  \hline 
        \workname~(Proto) & ${38.93_{\pm0.50}}$ & $\mathbf{54.52_{\pm0.47}}$  & $53.00_{\pm0.60}$  & $\mathbf{68.21_{\pm0.47}}$  & $34.29_{\pm0.50}$  &  $\mathbf{47.28_{\pm0.48}}$  \\ 
        \workname~(Matching) & $\mathbf{39.35_{\pm0.51}}$ & $50.25_{\pm0.46}$ & $\mathbf{54.44_{\pm0.59}}$  & $63.43_{\pm0.47}$  & $\mathbf{35.29_{\pm0.37}}$  & $42.18_{\pm0.44}$  \\ 
        \hline
        \multicolumn{7}{c}{16 Targets} \\ 
        \hline
        MatchingNet~\cite{vinyals2016matching}. & $29.77_{\pm0.30}$ & $40.35_{\pm0.33}$ &  $44.12_{\pm0.43}$ & $56.24_{\pm0.37}$ & $29.83_{\pm0.36}$ & $43.05_{\pm0.36}$ \\ 
        ProtoNet~\cite{snell2017prototypical} & $29.42_{\pm0.38}$ & $41.59_{\pm0.41}$  & $43.29_{\pm0.47}$ & $59.83_{\pm0.42}$ & $29.75_{\pm0.34}$ & $45.48_{\pm0.37}$ \\ 
        TransPropNet~\cite{liu2018learning} &  $20.00_{\pm0.00}$ &  $37.06_{\pm0.40}$ & $21.94_{\pm0.25}$ & $25.93_{\pm0.29}$ & $20.95_{\pm0.18}$ & $45.48_{\pm0.37}$ \\ 
        FTNET~\cite{dhillon2019baseline} &  $22.36_{\pm0.21}$  & $26.19_{\pm0.25}$ & $37.22_{\pm0.37}$ & $49.14_{\pm0.36}$ & $21.83_{\pm0.22}$ & $24.50_{\pm0.23}$ \\ 
        TP~\cite{bennequin2021bridging}  & $31.88_{\pm0.38}$  &  $45.50_{\pm0.37}$ & $45.54_{\pm0.49}$ & $62.32_{\pm0.38}$ & $28.92_{\pm0.37}$ & $44.92_{\pm0.39}$ \\ 
        AQP~\cite{aimen2022adversarial}.  & $32.78_{\pm0.18}$ & $46.10_{\pm0.31}$  & $46.27_{\pm0.41}$ & $61.77_{\pm0.49}$ & $29.98_{\pm0.43}$ & $45.88_{\pm0.40}$ \\
        PGADA (Proto)~\cite{jiang2022pgada} & $36.04_{\pm0.41}$  & $49.45_{\pm0.38}$ & $53.14_{\pm0.53}$ & $64.37_{\pm0.39}$ & $28.14_{\pm0.31}$ & $35.73_{\pm0.34}$ \\  
        PGADA (Matching)~\cite{jiang2022pgada} & $36.17_{\pm0.41}$ & $44.70_{\pm0.36}$ & $54.36_{\pm0.53}$  & $61.09_{\pm0.40}$ & $28.24_{\pm0.31}$   & $34.97_{\pm0.32}$   \\   \hline
        \workname~(Proto) & $\mathbf{39.58_{\pm0.46}}$ & $\mathbf{53.97_{\pm0.40}}$  & $56.37_{\pm0.55}$  & $\mathbf{67.83_{\pm0.40}}$  & $34.46_{\pm0.45}$  &  $\mathbf{47.81_{\pm0.41}}$  \\ 
        \workname~(Matching) & $39.50_{\pm0.46}$ & $50.11_{\pm0.40}$ & $\mathbf{57.12_{\pm0.55}}$  & $64.04_{\pm0.42}$  & $\mathbf{34.70_{\pm0.46}}$  &  $42.96_{\pm0.38}$  \\ \hline

    \end{tabular}
}
    \caption{Comparison of accuracy among different baselines and datasets.}
    \label{table:main_result}
\end{table*}

Following~\cite{bennequin2021bridging}, the average results with $95\%$ confidence interval from $2000$ runs are reported in the top-$1$ accuracy score. In addition, we conduct the tasks of $1$-shot and $5$-shot with $16$-target, i.e., $1$ or $5$ instances per class in the support set and $16$ instances in the query set, in CIFAR100, mini-ImageNet, and Tiered-ImageNet. Same with~\cite{jiang2022pgada}, we use a $4$-layer convolutional network as an embedding function $\phi$ on CIFAR100, ResNet18 for mini-ImageNet and Tiered-ImageNet. As a general adversarial training framework for few-shot learning, we combine \workname~with two classifiers, i.e., ProtoNet and MatchingNet, in the testing phase. As for our repairer $R$, we adopted an adjusted REDNET-like\cite{Mao2016Image} structure which is composed of a $4$-layer encoder-decoder structure and $2$ convolutional layers. 

\begin{itemize}
    \item \textbf{Matching Network}~\cite{vinyals2016matching} classifies each example $x_{q,j}$ in the query set $\mathcal{Q}$ to the class of the data point in the support set $\mathcal{S}$, with the highest consine similarity. Given a support example $x^c_{s,i}$ for class $c$, the probability $p (y =c | x_{q,j}) $ is defined as follows.
    \begin{equation*}
    p (y =c | x_{q,j}) =  \frac{exp(cos(\phi(x_{q,j)},\phi(x^{c}_{s,i})))}{\sum_{\mathbf{c} \in \mathcal{C}} exp(cos(\phi(x^{\mathbf{c}}_{s,i}),\phi(x_{q,j)}))}.
    \end{equation*}
    Note that Matching Network is targeting on the one-shot learning. In other words, there is only one data point $x^c_{s,i}$ in the support set of class $c$.

    \item \textbf{Prototypical Network}~\cite{snell2017prototypical} extends the Matching Network by introducing the prototype as there are multiple examples for each class in the support set. Formally, the prototype $\phi^{c}(x_{s,i})$ is defined as the mean vector of the support set in class $c$, i.e.,
    \begin{align*}
        \phi^{c}(x_{s}) = \frac{1}{|\mathcal{S}^c|} \sum_{x_{s,i} \in \mathcal{S}^c} (\phi(x_{s,i})) 
    \end{align*}
    Then, The Euclidean distance is adopted to estimate the probability.
    \begin{align*}
        p (y =c | x_{q,j}) = \frac{exp(-\Vert{ \phi(x_{q,j}) - \phi^{c}(x_{s})}\Vert_2^{2})}{\sum_{\mathbf{c} \in C}exp(-\Vert{\phi(x_{q,j}) - \phi^{\mathbf{c}}(x_{s})}\Vert_2^{2}}) 
    \end{align*}
\end{itemize}

The learning rate $\eta$, batch size $b$, and embedding dimention $d$ are set to $1e-3$, $128$, $128$, respectively. Besides, SGD with Adam optimizer~\cite{kingma2014adam} is adopted to train model in $200$ epochs with early stopping. Grid search is adopted to select the trade-off parameter in the objective function, i.e.,  $\lambda_1 = 1$ and $\lambda_2 =1$, $\beta = 0.5$ for best performance. We deploy the self-supervised learning on unlabeled data from testing set. Note that we adopt the transductive batch normalization~\cite{liu2018learning} on TransPropNet, FTNET, TP, PGADA, and our framework.

\subsection{Benchmark of RSQS}
\label{sec:quantitative}

We compare \workname~with seven state-of-the-art few-shot learning methods mentioned above. Since  \workname~and PGADA are model-agnostic adversarial alignment frameworks, we equip these two methods with two different classifiers, e.g., ProtoNet and MatchingNet.

\subsubsection{Main Results}
\label{sec:exp:main_result}

The results presented in Table \ref{table:main_result} demonstrate that \workname~consistently surpasses the performance of the first four baseline models (MatchingNet, ProtoNet, TransPropNet, and FTNET), with maximum performance improvements of $14.17\%$, $13.28\%$, $42.00\%$, and $27.78\%$, respectively.
This is because these baselines are not equipped to address the inherent distribution shift between the support set and the query set, which is realigned by the two adversarially-trained models in \workname. Though TP also leverages optimal transport to align the support and query sets, it still shows relatively weak performance compared to \workname, with $4.71\%$ accuracy loss at least among three datasets. The main reason for this phenomenon is that the transportation plan of TP is misguided by the small perturbations in the images, as proved in Theorem~\ref{thm:err}. 
Interestingly, the results for 8 targets and 16 targets show significant similarities. From a dataset perspective, \workname~ consistently outperforms on the CIFAR-100, mini-ImageNet, and Tiered-ImageNet datasets, with maximum performance improvements of $28.33\%$, $42.38\%$, and $23.08\%$, respectively. This can be primarily attributed to the complex nature of the RSQS task, which poses significant challenges and negatively impacts the performance of previous methods.

Moreover, our framework outperforms PGADA by at most $12.08\%$ accuracy, manifesting that \workname~could align the distribution of the query images to one of the support images by reconstructing the lost information caused by the perturbations.  Improvements between \workname~and PGADA are more significant in Tiered-ImageNet datasets, showing that PGADA can only handle some easy datasets,  with a relatively weak performance in large inter-domain bias environments. In addition, the accuracy increase on Tiered-ImageNet datasets is smaller than other datasets since the Tiered-ImageNet has a larger image size which is less affected by perturbations between support and query sets, As a result of the above two phenomena, the improvement in CIFAR100 is the most significant.  
Furthermore, our observations indicate that the integration of ProtoNet with \workname~enhances its classification robustness, particularly in 5-shot scenarios, when compared to MatchingNet. This improvement is primarily attributed to the prototypes in ProtoNet, which reduce bias by averaging the embedding vectors, thereby strengthening robustness. In our experiments, we observed a more significant improvement in accuracy in the 5-shot scenarios as compared to the 1-shot scenarios. This indicates that the classification accuracy tends to improve with an increasing number of repaired images.

\subsubsection{Case Study}
\label{sec:exp:case}

\begin{table*}[!ht]\small

\resizebox{2\columnwidth}{!}{
    \centering
\begin{tabular}{ccccccc}
\hline
\multicolumn{1}{c}{\multirow{2}{*}{Dataset}} &
  \multicolumn{2}{|c|}{CIFAR-100} &
  \multicolumn{2}{c|}{mini-ImageNet} &
  \multicolumn{2}{c}{Tiered-ImageNet}  \\ \cline{2-7} 
 & \multicolumn{1}{|c|}{1-Shot} &
  \multicolumn{1}{c|}{5-Shot} &
  \multicolumn{1}{c|}{1-Shot} &
  \multicolumn{1}{c|}{5-Shot} &
  \multicolumn{1}{c|}{1-Shot} &
  \multicolumn{1}{c}{5-Shot}  \\ \hline
\multicolumn{7}{c}{1 Shift} \\ \hline
\multicolumn{1}{c|}{MatchingNet~\cite{vinyals2016matching}} &
  \multicolumn{1}{c|}{$32.51_{\pm0.34}$} &
  \multicolumn{1}{c|}{$43.55_{\pm0.33}$} &
  \multicolumn{1}{c|}{$47.31_{\pm0.45}$} &
  \multicolumn{1}{c|}{$56.96_{\pm0.37}$} &
  \multicolumn{1}{c|}{$34.92_{\pm0.40}$} &
  \multicolumn{1}{c}{$52.79_{\pm0.38}$} \\
\multicolumn{1}{c|}{ProtoNet~\cite{snell2017prototypical}} &
  \multicolumn{1}{c|}{$31.90_{\pm0.33}$} &
  \multicolumn{1}{c|}{$45.73_{\pm0.35}$} &
  \multicolumn{1}{c|}{$45.91_{\pm0.44}$} &
  \multicolumn{1}{c|}{$61.15_{\pm0.37}$} &
  \multicolumn{1}{c|}{$34.23_{\pm0.38}$} &
  \multicolumn{1}{c}{$56.35_{\pm0.40}$} \\
\multicolumn{1}{c|}{TransPropNet~\cite{liu2018learning}} &
  \multicolumn{1}{c|}{$33.12_{\pm0.38}$} &
  \multicolumn{1}{c|}{$42.23_{\pm0.34}$} &
  \multicolumn{1}{c|}{$23.75_{\pm0.24}$} &
  \multicolumn{1}{c|}{$38.70_{\pm0.39}$} &
  \multicolumn{1}{c|}{$23.80_{\pm0.22}$} &
  \multicolumn{1}{c}{$29.09_{\pm0.33}$} \\
\multicolumn{1}{c|}{FTNET~\cite{dhillon2019baseline}} &
  \multicolumn{1}{c|}{$23.39_{\pm0.24}$} &
  \multicolumn{1}{c|}{$28.52_{\pm0.26}$} &
  \multicolumn{1}{c|}{$25.33_{\pm0.22}$} &
  \multicolumn{1}{c|}{$35.21_{\pm0.59}$} &
  \multicolumn{1}{c|}{$23.86_{\pm0.11}$} &
  \multicolumn{1}{c}{$28.33_{\pm0.24}$} \\
\multicolumn{1}{c|}{TP~\cite{bennequin2021bridging}} &
  \multicolumn{1}{c|}{$35.59_{\pm0.40}$} &
  \multicolumn{1}{c|}{$50.20_{\pm0.38}$} &
  \multicolumn{1}{c|}{$50.71_{\pm0.52}$} &
  \multicolumn{1}{c|}{$64.05_{\pm0.40}$} &
  \multicolumn{1}{c|}{$35.16_{\pm0.43}$} &
  \multicolumn{1}{c}{$57.87_{\pm0.42}$} \\
\multicolumn{1}{c|}{AQP~\cite{aimen2022adversarial}} &
  \multicolumn{1}{c|}{$36.28_{\pm0.40}$} &
  \multicolumn{1}{c|}{$50.38_{\pm0.33}$} &
  \multicolumn{1}{c|}{$51.87_{\pm0.40}$} &
  \multicolumn{1}{c|}{$64.80_{\pm0.57}$} &
  \multicolumn{1}{c|}{$34.19_{\pm0.52}$} &
  \multicolumn{1}{c}{$55.63_{\pm0.50}$} \\
\multicolumn{1}{c|}{PGADA~\cite{jiang2022pgada} } &
  \multicolumn{1}{c|}{$41.87_{\pm0.45}$} &
  \multicolumn{1}{c|}{$56.33_{\pm0.39}$} &
  \multicolumn{1}{c|}{$57.89_{\pm0.53}$} &
  \multicolumn{1}{c|}{$69.57_{\pm0.39}$} &
  \multicolumn{1}{c|}{$41.87_{\pm0.45}$} &
  \multicolumn{1}{c}{$56.33_{\pm0.39}$}  \\
\multicolumn{1}{c|}{\workname} &
  \multicolumn{1}{c|}{$\mathbf{42.95_{\pm0.48}}$} &
  \multicolumn{1}{c|}{$\mathbf{59.42_{\pm0.40}}$} &
  \multicolumn{1}{c|}{$\mathbf{59.34_{\pm0.54}}$} &
  \multicolumn{1}{c|}{$\mathbf{71.22_{\pm0.39}}$} &
  \multicolumn{1}{c|}{$\mathbf{42.95_{\pm0.48}}$} &
  \multicolumn{1}{c}{$\mathbf{59.42_{\pm0.40}}$}  \\ \hline

\multicolumn{7}{c}{2 Shifts} \\ \hline
\multicolumn{1}{c|}{MatchingNet~\cite{vinyals2016matching}} &
  \multicolumn{1}{c|}{$30.35_{\pm0.31}$} &
  \multicolumn{1}{c|}{$40.34_{\pm0.32}$} &
  \multicolumn{1}{c|}{$41.78_{\pm0.41}$} &
  \multicolumn{1}{c|}{$57.28_{\pm0.38}$} &
  \multicolumn{1}{c|}{$29.47_{\pm0.35}$} &
  \multicolumn{1}{c}{$43.34_{\pm0.36}$} \\
\multicolumn{1}{c|}{ProtoNet~\cite{snell2017prototypical}} &
  \multicolumn{1}{c|}{$29.25_{\pm0.30}$} &
  \multicolumn{1}{c|}{$41.73_{\pm0.33}$} &
  \multicolumn{1}{c|}{$43.78_{\pm0.43}$} &
  \multicolumn{1}{c|}{$59.87_{\pm0.37}$} &
  \multicolumn{1}{c|}{$29.67_{\pm0.34}$} &
  \multicolumn{1}{c}{$46.32_{\pm0.37}$} \\
\multicolumn{1}{c|}{TransPropNet~\cite{liu2018learning}} &
  \multicolumn{1}{c|}{$30.11_{\pm0.34}$} &
  \multicolumn{1}{c|}{$35.43_{\pm0.32}$} &
  \multicolumn{1}{c|}{$22.85_{\pm0.65}$} &
  \multicolumn{1}{c|}{$30.21_{\pm0.53}$} &
  \multicolumn{1}{c|}{$21.37_{\pm0.21}$} &
  \multicolumn{1}{c}{$28.99_{\pm0.21}$} \\
\multicolumn{1}{c|}{FTNET~\cite{dhillon2019baseline}} &
  \multicolumn{1}{c|}{$22.14_{\pm0.21}$} &
  \multicolumn{1}{c|}{$25.43_{\pm0.24}$} &
  \multicolumn{1}{c|}{$24.47_{\pm0.28}$} &
  \multicolumn{1}{c|}{$32.66_{\pm0.37}$} &
  \multicolumn{1}{c|}{$21.39_{\pm0.21}$} &
  \multicolumn{1}{c}{$23.37_{\pm0.22}$} \\
\multicolumn{1}{c|}{TP~\cite{bennequin2021bridging}} &
  \multicolumn{1}{c|}{$31.83_{\pm0.35}$} &
  \multicolumn{1}{c|}{$44.76_{\pm0.36}$} &
  \multicolumn{1}{c|}{$47.78_{\pm0.51}$} &
  \multicolumn{1}{c|}{$61.34_{\pm0.40}$} &
  \multicolumn{1}{c|}{$28.76_{\pm0.37}$} &
  \multicolumn{1}{c}{$45.47_{\pm0.38}$} \\
\multicolumn{1}{c|}{AQP~\cite{aimen2022adversarial}} &
  \multicolumn{1}{c|}{$32.21_{\pm0.21}$} &
  \multicolumn{1}{c|}{$45.09_{\pm0.83}$} &
  \multicolumn{1}{c|}{$48.25_{\pm0.13}$} &
  \multicolumn{1}{c|}{$62.13_{\pm0.29}$} &
  \multicolumn{1}{c|}{$29.61_{\pm0.35}$} &
  \multicolumn{1}{c}{$46.64_{\pm0.33}$} \\
\multicolumn{1}{c|}{PGADA~\cite{jiang2022pgada} } &
  \multicolumn{1}{c|}{$36.51_{\pm0.41}$} &
  \multicolumn{1}{c|}{$49.57_{\pm0.38}$} &
  \multicolumn{1}{c|}{$52.35_{\pm0.52}$} &
  \multicolumn{1}{c|}{$65.72_{\pm0.39}$} &
  \multicolumn{1}{c|}{$36.51_{\pm0.41}$} &
  \multicolumn{1}{c}{$49.57_{\pm0.38}$}  \\
\multicolumn{1}{c|}{\workname} &
  \multicolumn{1}{c|}{$\mathbf{39.91_{\pm0.47}}$} &
  \multicolumn{1}{c|}{$\mathbf{52.69_{\pm0.39}}$} &
  \multicolumn{1}{c|}{$\mathbf{56.02_{\pm0.55}}$} &
  \multicolumn{1}{c|}{$\mathbf{68.74_{\pm0.40}}$} &
  \multicolumn{1}{c|}{$\mathbf{39.91_{\pm0.47}}$} &
  \multicolumn{1}{c}{$\mathbf{52.69_{\pm0.39}}$}  \\ \hline
  
\multicolumn{7}{c}{4 Shifts} \\ \hline
\multicolumn{1}{c|}{MatchingNet~\cite{vinyals2016matching}} &
  \multicolumn{1}{c|}{$26.76_{\pm0.27}$} &
  \multicolumn{1}{c|}{$35.72_{\pm0.30}$} &
  \multicolumn{1}{c|}{$35.10_{\pm0.42}$} &
  \multicolumn{1}{c|}{$50.91_{\pm0.42}$} &
  \multicolumn{1}{c|}{$24.60_{\pm0.31}$} &
  \multicolumn{1}{c}{$32.04_{\pm0.39}$} \\
\multicolumn{1}{c|}{ProtoNet~\cite{snell2017prototypical}} &
  \multicolumn{1}{c|}{$26.17_{\pm0.28}$} &
  \multicolumn{1}{c|}{$36.05_{\pm0.30}$} &
  \multicolumn{1}{c|}{$39.83_{\pm0.46}$} &
  \multicolumn{1}{c|}{$56.47_{\pm0.43}$} &
  \multicolumn{1}{c|}{$27.11_{\pm0.38}$} &
  \multicolumn{1}{c}{$32.17_{\pm0.40}$} \\
\multicolumn{1}{c|}{TransPropNet~\cite{liu2018learning}} &
  \multicolumn{1}{c|}{$25.08_{\pm0.26}$} &
  \multicolumn{1}{c|}{$29.39_{\pm0.27}$} &
  \multicolumn{1}{c|}{$20.94_{\pm0.18}$} &
  \multicolumn{1}{c|}{$22.22_{\pm0.21}$} &
  \multicolumn{1}{c|}{$21.24_{\pm0.26}$} &
  \multicolumn{1}{c}{$26.52_{\pm0.34}$} \\
\multicolumn{1}{c|}{FTNET~\cite{dhillon2019baseline}} &
  \multicolumn{1}{c|}{$21.28_{\pm0.19}$} &
  \multicolumn{1}{c|}{$22.82_{\pm0.21}$} &
  \multicolumn{1}{c|}{$22.82_{\pm0.24}$} &
  \multicolumn{1}{c|}{$21.11_{\pm0.26}$} &
  \multicolumn{1}{c|}{$23.21_{\pm0.20}$} &
  \multicolumn{1}{c}{$20.71_{\pm0.29}$} \\
\multicolumn{1}{c|}{TP~\cite{bennequin2021bridging}} &
  \multicolumn{1}{c|}{$27.30_{\pm0.31}$} &
  \multicolumn{1}{c|}{$36.73_{\pm0.32}$} &
  \multicolumn{1}{c|}{$31.65_{\pm0.46}$} &
  \multicolumn{1}{c|}{$58.11_{\pm0.47}$} &
  \multicolumn{1}{c|}{$20.23_{\pm0.39}$} &
  \multicolumn{1}{c}{$20.31_{\pm0.39}$} \\
\multicolumn{1}{c|}{AQP~\cite{aimen2022adversarial}} &
  \multicolumn{1}{c|}{$32.21_{\pm0.21}$} &
  \multicolumn{1}{c|}{$45.09_{\pm0.83}$} &
  \multicolumn{1}{c|}{$48.25_{\pm0.13}$} &
  \multicolumn{1}{c|}{$62.13_{\pm0.29}$} &
  \multicolumn{1}{c|}{$21.61_{\pm0.35}$} &
  \multicolumn{1}{c}{$20.41_{\pm0.31}$} \\
\multicolumn{1}{c|}{PGADA~\cite{jiang2022pgada}} &
  \multicolumn{1}{c|}{$28.85_{\pm0.32}$} &
  \multicolumn{1}{c|}{$37.10_{\pm0.33}$} &
  \multicolumn{1}{c|}{$45.26_{\pm0.48}$} &
  \multicolumn{1}{c|}{$57.31_{\pm0.39}$} &
  \multicolumn{1}{c|}{$20.43_{\pm0.26}$} & 
  \multicolumn{1}{c}{$23.10_{\pm0.33}$}  \\
\multicolumn{1}{c|}{\workname} &
  \multicolumn{1}{c|}{$\mathbf{35.54_{\pm0.44}}$} &
  \multicolumn{1}{c|}{$\mathbf{47.48_{\pm0.40}}$} &
  \multicolumn{1}{c|}{$\mathbf{51.33_{\pm0.55}}$} &
  \multicolumn{1}{c|}{$\mathbf{65.08_{\pm0.41}}$} &
  \multicolumn{1}{c|}{$\mathbf{27.73_{\pm0.42}}$} &
  \multicolumn{1}{c}{$\mathbf{32.38_{\pm0.41}}$}  \\
 \hline
\end{tabular}
}
\caption{The result of case studies under multiple shifts settings. Best results are marked as bold.}
\label{table:case_study}
\end{table*}

To further investigate the robustness and generality of the repairer in \workname, we conduct a case study of multiple shifts in this section. Compared to Sec.~\ref{sec:exp:main_result}, we study the effects of the number of shifts in this section.
In particular,  the instance in each meta-test task receives various perturbations, i.e., 1, 2, 4.  Note that we want we only present the 16 targets results in this paper since in the other two levels, e.g., 8 targets results are consistently similar. In addition,  we also only present the results of \workname~with ProtoNet as \workname~with MatchingNet shows similar ones. In Table \ref{table:case_study}, it shows that, under different perturbations, \workname~outperforms the first four baselines by $10.16\%$, $12.55\%$, $11.75\%$, $8.84\%$ for 1-shot setting and $11.48\%$, $16.58\%$, $10.71\%$, $10.41\%$ for 5-shot setting.  In addition, PGADA is degraded by the random shifts by $4.62\%$, $4.16\%$, and $5.98\%$ in the three datasets on average. 
According to these results, \workname~shows a robust performance under the inter-domain bias and intra-domain variance. Also, the repairer network is able to encode the distribution shifts by learning from repairing perturbed images and recovering crucial information that is damaged.

The \workname~method has shown notable improvements compared to PGADA across three datasets, boasting an average accuracy increase of $8.82\%$ for the 4-perturbation task. While improvements were also discernible in the 1 or 2 perturbation tasks, they were less substantial due to the adverse effects caused by the increased number of perturbations. Particularly for low-resolution images, such as those extracted from the CIFAR-100 dataset, the task involving multiple perturbations, i.e., 4 perturbations, led to the most significant decline in accuracy. Taking these observations into account, they highlight the effectiveness of the \workname~method in repairing lost information, even in those severe cases.  Furthermore, the \workname~method displayed the most significant improvement on the Tiered-ImageNet dataset, achieving a $10.64\%$ enhancement in accuracy across multiple perturbation choices. It also showed improvements of $5.27\%$ and $6.40\%$ on CIFAR100 and mini-ImageNet, respectively. These results suggest that support-query alignment can effectively denoise images. A comprehensive analysis of these findings will be presented in Section\ref{sec:ablation}. An interesting observation is that \workname~displays greater robustness than all the baselines in a 1-Shot situation. For instance, on CIFAR-100, when comparing a performance similar to \workname, i.e., PGADA, we observe that \workname~decreases from $42.95\%$ to $35.54\%$, while PGADA drops around $13.02\%$. This observation suggests that our repair network is effective in mitigating more noise, even in scenarios with limited instances.

\subsection{ Ablation Studies of \workname}
\label{sec:ablation}

\begin{table*}[!ht]\small

\resizebox{2\columnwidth}{!}{
    \centering
\begin{tabular}{ccccccc}
\hline
\multicolumn{1}{c}{\multirow{2}{*}{Dataset}} &
  \multicolumn{2}{|c|}{CIFAR-100} &
  \multicolumn{2}{c|}{mini-ImageNet} &
  \multicolumn{2}{c}{Tiered-ImageNet}  \\ \cline{2-7} 
 & \multicolumn{1}{|c|}{1-Shot} &
  \multicolumn{1}{c|}{5-Shot} &
  \multicolumn{1}{c|}{1-Shot} &
  \multicolumn{1}{c|}{5-Shot} &
  \multicolumn{1}{c|}{1-Shot} &
  \multicolumn{1}{c}{5-Shot}  \\ \hline
\multicolumn{7}{c}{ Ablation studies. } \\ \hline
\multicolumn{1}{c|}{$Fixed\ G$} &
  \multicolumn{1}{c|}{$37.56_{\pm0.45}$} &
  \multicolumn{1}{c|}{$52.01_{\pm0.39}$} &
  \multicolumn{1}{c|}{$54.22_{\pm0.57}$} &
  \multicolumn{1}{c|}{$66.27_{\pm0.40}$} &
  \multicolumn{1}{c|}{$31.19_{\pm0.41}$} &
  \multicolumn{1}{c}{$42.54_{\pm0.40}$} \\ 
\multicolumn{1}{c|}{w./o. $R$} &
  \multicolumn{1}{c|}{$27.47_{\pm0.36}$} &
  \multicolumn{1}{c|}{$35.05_{\pm0.39}$} &
  \multicolumn{1}{c|}{$44.12_{\pm0.43}$} &
  \multicolumn{1}{c|}{$62.33_{\pm0.38}$} &
  \multicolumn{1}{c|}{$26.73_{\pm0.26}$} &
  \multicolumn{1}{c}{$37.92_{\pm0.32}$}  \\
\multicolumn{1}{c|}{w./o. OT} &
  \multicolumn{1}{c|}{$36.04_{\pm0.37}$} &
  \multicolumn{1}{c|}{$50.76_{\pm0.38}$} &
  \multicolumn{1}{c|}{$46.46_{\pm0.45}$} &
  \multicolumn{1}{c|}{$65.61_{\pm0.39}$} &
  \multicolumn{1}{c|}{$29.03_{\pm0.33}$} &
  \multicolumn{1}{c}{$44.46_{\pm0.38}$} \\

\multicolumn{1}{c|}{w./o. SSL} &
  \multicolumn{1}{c|}{$37.58_{\pm0.44}$} &
  \multicolumn{1}{c|}{$51.02_{\pm0.39}$} &
  \multicolumn{1}{c|}{$51.45_{\pm0.51}$} &
  \multicolumn{1}{c|}{$63.96_{\pm0.40}$} &
  \multicolumn{1}{c|}{$35.24_{\pm0.47}$} &
  \multicolumn{1}{c}{$48.73_{\pm0.42}$} \\

  \hline

  \multicolumn{7}{c}{  Deep Analysis  } \\  \hline
  \multicolumn{1}{c|}{TP} &
  \multicolumn{1}{c|}{$31.88_{\pm0.38}$} &
  \multicolumn{1}{c|}{$45.50_{\pm0.37}$} &
  \multicolumn{1}{c|}{$45.54_{\pm0.49}$} &
  \multicolumn{1}{c|}{$62.32_{\pm0.38}$} &
  \multicolumn{1}{c|}{$28.92_{\pm0.37}$} &
  \multicolumn{1}{c}{$44.92_{\pm0.39}$} \\
  \multicolumn{1}{c|}{TP + R} &
  \multicolumn{1}{c|}{$32.03_{\pm0.36}$} &
  \multicolumn{1}{c|}{$46.13_{\pm0.40}$} &
  \multicolumn{1}{c|}{$48.58_{\pm0.53}$} &
  \multicolumn{1}{c|}{$64.25_{\pm0.40}$} &
  \multicolumn{1}{c|}{$28.52_{\pm0.39}$} &
  \multicolumn{1}{c}{$41.22_{\pm0.38}$} \\ 
  \multicolumn{1}{c|}{Encoding shift to $\phi$} &
  \multicolumn{1}{c|}{$34.56_{\pm0.38}$} &
  \multicolumn{1}{c|}{$45.98_{\pm0.38}$} &
  \multicolumn{1}{c|}{$49.37_{\pm0.50}$} &
  \multicolumn{1}{c|}{$62.55_{\pm0.39}$} &
  \multicolumn{1}{c|}{$24.26_{\pm0.26}$} &
  \multicolumn{1}{c}{$29.11_{\pm0.29}$} \\        
    \hline
 \multicolumn{1}{c|}{\workname}   &
  \multicolumn{1}{c|}{$\mathbf{39.58_{\pm0.46}}$}   & 
 \multicolumn{1}{c|}{$\mathbf{53.97_{\pm0.40}}$}  & 
 \multicolumn{1}{c|}{$\mathbf{56.37_{\pm0.55}}$}           & 
 \multicolumn{1}{c|}{$\mathbf{67.83_{\pm0.40}}$}  & 
 \multicolumn{1}{c|}{$\mathbf{40.16_{\pm0.51}}$}           &  
 \multicolumn{1}{c}{$\mathbf{54.28_{\pm0.43}}$}  \\ \hline
\end{tabular}
}
\caption{The results of ablation studies and deep Analysis of Repairer. Best results are marked as bold.}
\label{table:ablation}
\end{table*}

We conduct ablation studies to verify whether each component and designation in \workname~is contributing to the joint performance and the irreplaceability of the repairer in \workname~in Table \ref{table:ablation}. Similar to Sec.~\ref{sec:exp:case}, we only present the results of \workname~with ProtoNet as \workname~with 16 targets. Since the detailed analysis of the data augmentation methods is presented in~\cite{jiang2022pgada}, we only discuss the other essential components of \workname.

\subsubsection{Effect of Generator and Repairer.} 
We compare \workname~with its two variants, including i) \textit{fixed $G$}, fixing the parameters of the generator, and ii) \textit{w/o Repairer}, removing the repairer from the framework. Shown in Table~\ref{table:ablation}, when we fix the parameters of the generator, we can observe that the performance of \workname~drops by $2.23\%$, $2.15\%$,  and $11.74\%$ under the three datasets, respectively. It highlights how our trainable generator can extract the information from original images $x$ in order to produce meaningful perturbed images $x_p$. Furthermore, the \workname~method, when implemented without the repairer, displays a consistent decrease in performance: accuracy diminishes by $10.03\%$, $6.38\%$, and $9.08\%$. This outcome underscores the effectiveness of the repairer's adversarial-based alignment in extracting the inherent information from the original data using a robust embedding model. It effectively mitigates inter-domain bias and intra-domain variance, thereby facilitating improved alignment.

\subsubsection{Effect of Regularized Optimal Transportation.}
Here, we examine the effect of regularized optimal transportation (OT), a key component of few-shot learning under the support-query shift during the evaluation phase. Table~\ref{table:ablation} shows that OT significantly improves the model's capability by $3.36\%$, $3.75\%$, $4.31\%$ and $3.61\%$ averagely under the four shifts since it aligns the distributions of the support and query set. Even though the performance of \workname~drops as we remove OT, \workname~still outperforms the baseline method TP \cite{bennequin2021bridging} by $4.03\%$, which also utilizes OT but does not take into account the robustness of the image representation. 
This observation echoes the motivation of this work (Theorem~\ref{thm:err}), namely, that the small perturbations misguide the optimal transportation plan. This outcome manifests that a robust embedding model alleviates the support-query shift in few-shot learning, echoing our motivation, i.e., obtaining clean embeddings.

\subsubsection{Effect of Self-Supervised Learning.}
We also evaluated the impact of self-supervised learning. As Table~\ref{table:ablation} illustrates, \workname~experiences improvements of $2.48\%$, $4.39\%$, and $5.23\%$ on the three datasets, respectively. By incorporating contrastive loss, \workname~utilizes the structural information of the unlabeled data during the training phase to boost model performance. The results further demonstrate how \workname~and self-supervised learning can be combined to extract information from both the training and testing sets, thereby enhancing the generalizability of the embedding model. It's worth noting that in RSQS, the unlabeled data does not contribute to a significant improvement as compared to the repairer network.

\subsection{Deep Analysis of \workname}
In this section, we also conduct two deep analyses, i.e., repairer and hyperparameter analysis for a comprehensive discussion of \workname.

\subsubsection{Repairer Analysis.}
Furthermore, we delve into a detailed analysis of our repairer $R$, which aims to alleviate inter-domain bias by aligning the distribution. Firstly, we validate the necessity of the repairer by encoding the shifts into the feature extractor, specifically by training the model $\phi$ with perturbed images. This approach addresses potential questions regarding the differences between these two practices. As shown in Table~\ref{table:ablation}, the results indicate that performance improves by $2.68\%$ with the separate repairing function, suggesting that an independent model capacity may be a better option for denoising input images. Secondly, we attach the repairer to TP to evaluate its potential benefits on other frameworks. The performance improvement of the repairer on TP is minor, with only a $0.15\%$ increase on CIFAR-100, indicating a marginal advantage. However, when the repairer is integrated with \workname, the improvement is more significant at $1.83\%$. This evidence suggests that our repairer network is specifically designed for, and performs optimally within our framework. Lastly. with the designation of the dual adversarial distribution alignment, we observe that the repairer could be leveraged better, which means that the implicit alignment induced by the generator-encoding co-training and the explicit alignment brought by the repairer jointly achieve a more robust representation of images.

\subsubsection{Hyperparameter Analysis.}
We also conduct a hyperparameter analysis of \workname. In particular, we use various dimensions i.e., 32, 64, 128, of each layer in repair networks across three datasets. As illustrated in Figure~\ref{fig:hyper}, mini-ImageNet demonstrates the most exceptional performance, a phenomenon that can be attributed to using the highest resolution amongst its counterparts. Moreover, we observe a significant enhancement in performance across all three datasets with an increase in dimensionality. For instance, in the CIFAR100 dataset, the accuracy of the 5-shot task elevates from $44.11\%$ at a $32$-dimension level to $53.92\%$ when the dimensionality is increased to $128$. However, the performance gains are less significant when the dimension is increased from $64$ to $128$, particularly when compared to the increase from $32$ to $64$. The primary reason for this is that the repair network is capable of extracting sufficient information at a dimensionality of $64$.

\begin{figure}[t]
     \centering
     \begin{subfigure}[b]{.9\linewidth}
         \centering
         \includegraphics[width=\linewidth]{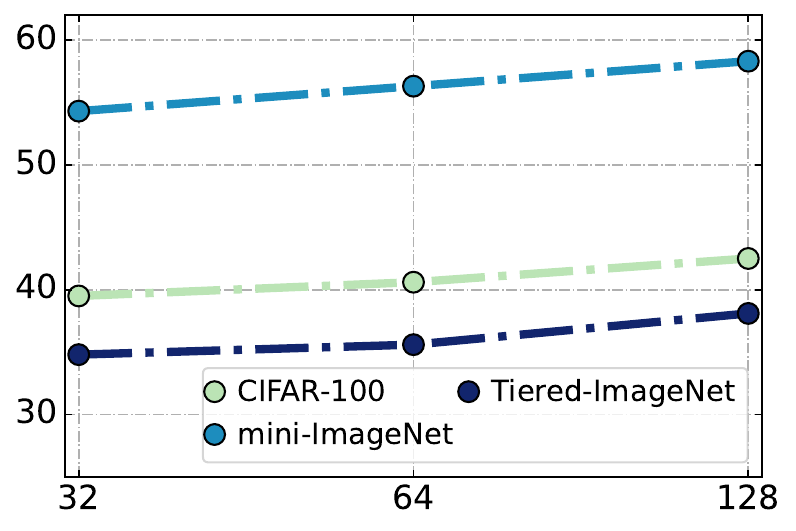}
         \vspace{-5pt}
         \caption{ Accuracy Comparison in 1-Shot}
         \label{fig:y equals x}
     \end{subfigure}
     \begin{subfigure}[b]{.9\linewidth}
         \centering
         \includegraphics[width=\linewidth]{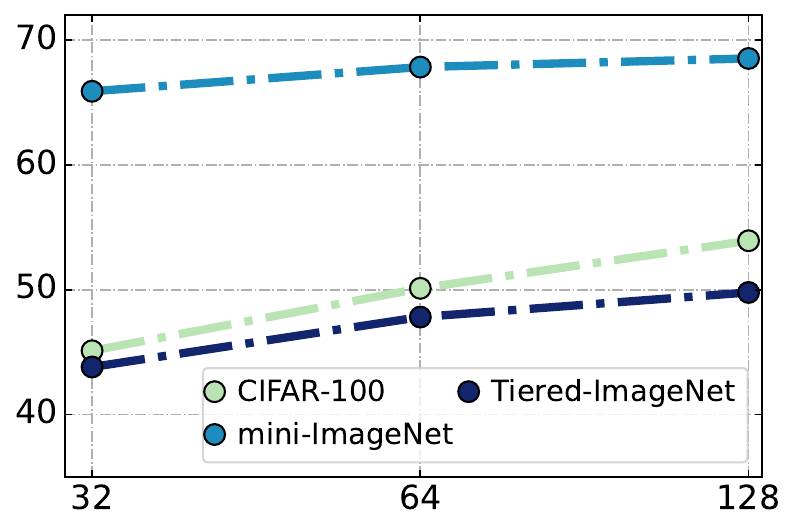}
         \caption{ Accuracy Comparison in 5-Shot}
         \label{fig:three sin x}
     \end{subfigure}
     \vspace{-5pt}
    \caption{Hyperparameter Analysis among 3 datasets.}
     \vspace{-10pt}

  \label{fig:hyper}
\end{figure}

\section{Related Work}
\label{sec:related}
We present our related work from the following perspectives including (1) few-shot learning, (2) support-query shift few-shot learning, and (3) data augmentation and optimal transport.
\subsection{Few-shot Learning} 
The study of few-shot Learning mainly can be divided into three categories: hallucination based, optimization based, and metric based methods~\cite{luo2021few,parnami2022learning}.
The hallucination based methods synthesize representations~\cite{hariharan2017low,luo2021few,wang2018low} and use Generative Adversarial Networks (GANs)~\cite{antoniou2017data,li2020adversarial} to generate the deficient labeled data.  The optimization based methods address the few-shot problem by training an adaptable model~\cite{finn2017model,yao2019hierarchically}. The purpose of such methods is to obtain a well-trained initial model which only needs a few optimizations to adopt few-shot tasks. However, optimization based methods are considered that are challenged to learn effective embeddings and being easier infected with oversampling of meta-training classes~\cite{rusu2018meta}. Recent research~\cite{Lee_2022_WACV} scaled the gradient norms of the backbone in a task-wise fashion, improved the task-specific knowledge learning in MAML\cite{finn2017model}
In contrast with optimization based methods, the metric based methods aim to learn a classifier to assess the similarity between examples~\cite{vinyals2016matching,snell2017prototypical,huang2021pseudo}, which is most relevant to our work. MatchingNet~\cite{vinyals2016matching} and ProtoNet~\cite{snell2017prototypical} respectively used pairwise metrics and class-wise metrics to select the closest label of the query set samples according to the support set, and the RelationNet~\cite{zhuang2018relationnet} trying to learn a deep distance metric directly. Sung et al.~\cite{sung2018learning} model the non-linear relation between class representations and queries by neural networks. 
The latest research in this category~\cite{Afrasiyabi_2022_CVPR} extracts sets of feature vectors rather than single feature for each image and uses a set-to-set matching metric as the distance metric.

\subsection{Support-Query Shift Few-shot Learning}
Support query shift is a nascent area of few-shot learning research, and several studies have been performed well to address this type of problem. 
Bennequin et al.~\cite{bennequin2021bridging} utilized optimal transport to align the query set and support set to tackle this issue. However, that the perturbation in the pictures would mislead the best transport plan and result in an undesirable performance.
To solve this problem, PGADA~\cite{jiang2022pgada} combined regularized optimal transportation and adversarial generator to produce challenging examples for self-supervised learning to solve this problem.
In addition, Aimen et al.~\cite{aimen2022adversarial} also sought to generate challenging examples for the meta-training of a more robust model, but it used the projection method rather than an adversarial generator. A similar study line is cross-domain few-shot learning~\cite{chen2019closer,phoo2020self,li2022cross}, which also considers the distribution shift in few-shot learning. However, it is different to the support-query shift since the domain shift happened between the meta-training and meta-testing phases instead of the support and query sets. 

\subsection{Data Augmentation and Optimal Transport}
Since the dominate challenging of support query few-shot learning is to overcome the distribution shift between the support set and query set, there are mainly two types of methods in this study line, i.e., data augmentation and optimal transport~\cite{jiang2022pgada}.
On one hand, data augmentation is a classical method in machine learning to enrich the dataset\cite{khalifa2021comprehensive}, including classical data augmentation and adversarial based augmentation.
The classical data augmentation ~\cite{zhang2017mixup,yun2019cutmix} usually transforms a single image to its perturbations, some research like~\cite{chen2020simple,xie2020self} further estimating the pairwise similarity of augmented data with self-training schemes.
In contrast, adversarial based augmentation~\cite{samangouei2018defense,theagarajan2019shieldnets} generates new data points that are excluded from the original dataset. Those methods are able to generate perturbed examples from clean data that the model misclassifies as new training data to make the model more robust. On the other hand, the optimal transport is widely used to in machine learning studies such as domain adaptation because it is easy to align different distributions, using optimal transport~\cite{courty2016optimal,wang2021zero}. For example, Wang et al.~\cite{wang2021zero} utilised optimal transport between model-generated feature and real features for zero-shot recognition, Nguyen et al.~\cite{nguyen2021most} tried to solve multi-source domain adaptation problem by optimal transport and imitation learning.
In particular, a regularized unsupervised optimal transportation method is proposed~\cite{courty2016optimal} to align the representations between the source and target domains. In order to reduce the computation complexity, Sinkhorn's iterative algorithm and its variants~\cite{cuturi2013sinkhorn} are used to efficiently solve the optimal transport, which applies iterative Bregman projection.

\section{Conclusion}
In this paper, we study a more practical and complex situation of support-query shifts, RSQS, from inter-domain bias and intra-domain variance. To relieve such shifts, we propose \emph{\textbf{Du}al Adversarial \textbf{Al}ignment Framework (DuaL)} to solve the support-query shift in few-shot learning. Our key idea is to generate perturbed images in both support and query sets and then train on these perturbed data to derive a more robust embedding model and alleviate the misestimation of optimal transportation. In addition, a negative entropy regularization is introduced to obtain a smooth transportation plan. A benchmark is also conducted of RSQS, experiment results manifest that \workname~significantly outperforms relatively 7 baselines in CIFAR100, mini-ImageNet, and Tiered-ImageNet. Future works include applying \workname~to other computer vision tasks and incorporating it with other data augmentation schemes.

\bibliographystyle{IEEEtran}
\bibliography{myref}

\begin{thebibliography}{10}
\providecommand{\url}[1]{#1}
\csname url@samestyle\endcsname
\providecommand{\newblock}{\relax}
\providecommand{\bibinfo}[2]{#2}
\providecommand{\BIBentrySTDinterwordspacing}{\spaceskip=0pt\relax}
\providecommand{\BIBentryALTinterwordstretchfactor}{4}
\providecommand{\BIBentryALTinterwordspacing}{\spaceskip=\fontdimen2\font plus
\BIBentryALTinterwordstretchfactor\fontdimen3\font minus
  \fontdimen4\font\relax}
\providecommand{\BIBforeignlanguage}[2]{{%
\expandafter\ifx\csname l@#1\endcsname\relax
\typeout{** WARNING: IEEEtran.bst: No hyphenation pattern has been}%
\typeout{** loaded for the language `#1'. Using the pattern for}%
\typeout{** the default language instead.}%
\else
\language=\csname l@#1\endcsname
\fi
#2}}
\providecommand{\BIBdecl}{\relax}
\BIBdecl

\bibitem{caldas2018leaf}
S.~Caldas, S.~M.~K. Duddu, P.~Wu, T.~Li, J.~Kone{\v{c}}n{\`y}, H.~B. McMahan,
  V.~Smith, and A.~Talwalkar, ``Leaf: A benchmark for federated settings,''
  \emph{In NeurIPS}, 2019.

\bibitem{shuai2022balancefl}
X.~Shuai, Y.~Shen, S.~Jiang, Z.~Zhao, Z.~Yan, and G.~Xing, ``Balancefl:
  Addressing class imbalance in long-tail federated learning,'' \emph{In IEEE
  IPSN}, pp. 271--284, 2022.

\bibitem{ye2022licam}
R.~Ye, Y.~Guo, X.~Shuai, R.~Ye, S.~Jiang, and H.~Jiang, ``Licam: Long-tailed
  instance segmentation with real-time classification accuracy monitoring,''
  \emph{Journal of Circuits, Systems and Computers}, 2022.

\bibitem{garcia2017few}
V.~Garcia and J.~Bruna, ``Few-shot learning with graph neural networks,'' in
  \emph{In ICLR}, 2018.

\bibitem{jiang2022pgada}
S.~Jiang, W.~Ding, H.-W. Chen, and M.-S. Chen, ``Pgada: Perturbation-guided
  adversarial alignment for few-shot learning under the support-query shift,''
  \emph{In PAKDD}, pp. 3--15, 2022.

\bibitem{vinyals2016matching}
O.~Vinyals, C.~Blundell, T.~Lillicrap, D.~Wierstra \emph{et~al.}, ``Matching
  networks for one shot learning,'' \emph{In NeurIPS}, vol.~29, 2016.

\bibitem{snell2017prototypical}
J.~Snell, K.~Swersky, and R.~Zemel, ``Prototypical networks for few-shot
  learning,'' \emph{In NeruIPS}, vol.~30, 2017.

\bibitem{finn2017model}
C.~Finn, P.~Abbeel, and S.~Levine, ``Model-agnostic meta-learning for fast
  adaptation of deep networks,'' in \emph{ICML}.\hskip 1em plus 0.5em minus
  0.4em\relax PMLR, 2017.

\bibitem{bennequin2021bridging}
E.~Bennequin, V.~Bouvier, M.~Tami, A.~Toubhans, and C.~Hudelot, ``Bridging
  few-shot learning and adaptation: New challenges of support-query shift,''
  \emph{ECML-PKDD}, 2021.

\bibitem{song2022comprehensive}
Y.~Song, T.~Wang, S.~K. Mondal, and J.~P. Sahoo, ``A comprehensive survey of
  few-shot learning: Evolution, applications, challenges, and opportunities,''
  \emph{arXiv preprint arXiv:2205.06743}, 2022.

\bibitem{zhong2020random}
Z.~Zhong, L.~Zheng, G.~Kang, S.~Li, and Y.~Yang, ``Random erasing data
  augmentation,'' in \emph{AAAI}, vol.~34, no.~07, 2020.

\bibitem{devries2017improved}
T.~DeVries and G.~W. Taylor, ``Improved regularization of convolutional neural
  networks with cutout,'' \emph{arXiv preprint arXiv:1708.04552}, 2017.

\bibitem{nakamura2019revisiting}
A.~Nakamura and T.~Harada, ``Revisiting fine-tuning for few-shot learning,''
  \emph{arXiv preprint arXiv:1910.00216}, 2019.

\bibitem{cai2020cross}
J.~Cai and S.~M. Shen, ``Cross-domain few-shot learning with meta
  fine-tuning,'' \emph{arXiv preprint arXiv:2005.10544}, 2020.

\bibitem{li2019large}
A.~Li, T.~Luo, Z.~Lu, T.~Xiang, and L.~Wang, ``Large-scale few-shot learning:
  Knowledge transfer with class hierarchy,'' in \emph{CVPR}, 2019, pp.
  7212--7220.

\bibitem{wang2020generalizing}
Y.~Wang, Q.~Yao, J.~T. Kwok, and L.~M. Ni, ``Generalizing from a few examples:
  A survey on few-shot learning,'' \emph{ACM computing surveys (csur)},
  vol.~53, no.~3, pp. 1--34, 2020.

\bibitem{li2021learning}
J.~Li, Z.~Wang, and X.~Hu, ``Learning intact features by erasing-inpainting for
  few-shot classification,'' in \emph{AAAI}, vol.~35, no.~9, 2021, pp.
  8401--8409.

\bibitem{gao2018low}
H.~Gao, Z.~Shou, A.~Zareian, H.~Zhang, and S.-F. Chang, ``Low-shot learning via
  covariance-preserving adversarial augmentation networks,'' \emph{In NeurIPS},
  vol.~31, 2018.

\bibitem{ma2020metacgan}
Y.~Ma, G.~Zhong, Y.~Wang, and W.~Liu, ``Metacgan: A novel gan model for
  generating high quality and diversity images with few training data,'' in
  \emph{In IJCNN}.\hskip 1em plus 0.5em minus 0.4em\relax IEEE, 2020, pp. 1--7.

\bibitem{kim2019variational}
J.~Kim, T.-H. Oh, S.~Lee, F.~Pan, and I.~S. Kweon, ``Variational
  prototyping-encoder: One-shot learning with prototypical images,'' \emph{In
  NeurIPS}, pp. 9462--9470, 2019.

\bibitem{schwartz2022baby}
E.~Schwartz, L.~Karlinsky, R.~Feris, R.~Giryes, and A.~Bronstein, ``Baby steps
  towards few-shot learning with multiple semantics,'' \emph{Pattern
  Recognition Letters}, vol. 160, pp. 142--147, 2022.

\bibitem{aimen2022adversarial}
A.~Aimen, B.~Ladrecha, and N.~C. Krishnan, ``Adversarial projections to tackle
  support-query shifts in few-shot meta-learning,'' in \emph{First Conference
  on Automated Machine Learning (Late-Breaking Workshop)}, 2022.

\bibitem{courty2016optimal}
N.~Courty, R.~Flamary, D.~Tuia, and A.~Rakotomamonjy, ``Optimal transport for
  domain adaptation,'' \emph{IEEE TPAMI}, 2016.

\bibitem{zhao2020maximum}
L.~Zhao, T.~Liu, X.~Peng, and D.~Metaxas, ``Maximum-entropy adversarial data
  augmentation for improved generalization and robustness,'' \emph{In NeurIPS},
  vol.~33, pp. 14\,435--14\,447, 2020.

\bibitem{gong2021maxup}
C.~Gong, T.~Ren, M.~Ye, and Q.~Liu, ``Maxup: Lightweight adversarial training
  with data augmentation improves neural network training,'' in \emph{CVPR},
  2021, pp. 2474--2483.

\bibitem{simonyan2014very}
K.~Simonyan and A.~Zisserman, ``Very deep convolutional networks for
  large-scale image recognition,'' in \emph{ICLR}, 2015.

\bibitem{yun2019cutmix}
S.~Yun, D.~Han, S.~J. Oh, S.~Chun, J.~Choe, and Y.~Yoo, ``Cutmix:
  Regularization strategy to train strong classifiers with localizable
  features,'' in \emph{ICCV}, 2019, pp. 6023--6032.

\bibitem{zhang2017mixup}
H.~Zhang, M.~Cisse, Y.~N. Dauphin, and D.~Lopez-Paz, ``mixup: Beyond empirical
  risk minimization,'' in \emph{ICLR}, 2017.

\bibitem{samangouei2018defense}
P.~Samangouei, M.~Kabkab, and R.~Chellappa, ``Defense-gan: Protecting
  classifiers against adversarial attacks using generative models,'' in
  \emph{ICLR}, 2018.

\bibitem{huang2018auggan}
S.-W. Huang, C.-T. Lin, S.-P. Chen, Y.-Y. Wu, P.-H. Hsu, and S.-H. Lai,
  ``Auggan: Cross domain adaptation with gan-based data augmentation,'' in
  \emph{ECCV}, 2018, pp. 718--731.

\bibitem{cuturi2013sinkhorn}
M.~Cuturi, ``Sinkhorn distances: Lightspeed computation of optimal transport,''
  \emph{In NeurIPS}, vol.~26, 2013.

\bibitem{ian2014generative}
I.~Goodfellow, J.~Pouget-Abadie, M.~Mirza, B.~Xu, D.~Warde-Farley, S.~Ozair,
  A.~Courville, and Y.~Bengio, ``Generative adversarial nets,'' \emph{In
  NeurIPS}, vol.~27, p. 2672–2680, 2014.

\bibitem{wang2018low}
Y.-X. Wang, R.~Girshick, M.~Hebert, and B.~Hariharan, ``Low-shot learning from
  imaginary data,'' in \emph{CVPR}, 2018, pp. 7278--7286.

\bibitem{phoo2020self}
C.~P. Phoo and B.~Hariharan, ``Self-training for few-shot transfer across
  extreme task differences,'' in \emph{In ICLR}, 2020.

\bibitem{bottou2012stochastic}
L.~Bottou, ``Stochastic gradient descent tricks,'' in \emph{Neural networks:
  Tricks of the trade}.\hskip 1em plus 0.5em minus 0.4em\relax Springer, 2012,
  pp. 421--436.

\bibitem{antoniou2017data}
A.~Antoniou, A.~Storkey, and H.~Edwards, ``Data augmentation generative
  adversarial networks,'' in \emph{International Conference on Learning
  Representations}, 2017.

\bibitem{chen2020simple}
T.~Chen, S.~Kornblith, M.~Norouzi, and G.~Hinton, ``A simple framework for
  contrastive learning of visual representations,'' in \emph{ICML}.\hskip 1em
  plus 0.5em minus 0.4em\relax PMLR, 2020, pp. 1597--1607.

\bibitem{he2015convolutional}
K.~He and J.~Sun, ``Convolutional neural networks at constrained time cost,''
  in \emph{CVPR}, 2015, pp. 5353--5360.

\bibitem{hendrycks2019benchmarking}
D.~Hendrycks and T.~Dietterich, ``Benchmarking neural network robustness to
  common corruptions and perturbations,'' in \emph{International Conference on
  Learning Representations}, 2019.

\bibitem{krizhevsky2009learning}
A.~Krizhevsky, G.~Hinton \emph{et~al.}, ``Learning multiple layers of features
  from tiny images,'' 2009.

\bibitem{zhang2020adaptive}
M.~M. Zhang, H.~Marklund, N.~Dhawan, A.~Gupta, S.~Levine, and C.~Finn,
  ``Adaptive risk minimization: A meta-learning approach for tackling group
  shift,'' \emph{arXiv preprint arXiv:2007.02931}, 2020.

\bibitem{triantafillou2019meta}
E.~Triantafillou, T.~Zhu, V.~Dumoulin, P.~Lamblin, U.~Evci, K.~Xu, R.~Goroshin,
  C.~Gelada, K.~Swersky, P.-A. Manzagol \emph{et~al.}, ``Meta-dataset: A
  dataset of datasets for learning to learn from few examples,'' \emph{arXiv
  preprint arXiv:1903.03096}, 2019.

\bibitem{ren2018meta}
M.~Ren, E.~Triantafillou, S.~Ravi, J.~Snell, K.~Swersky, J.~B. Tenenbaum,
  H.~Larochelle, and R.~S. Zemel, ``Meta-learning for semi-supervised few-shot
  classification,'' \emph{arXiv preprint arXiv:1803.00676}, 2018.

\bibitem{liu2018learning}
Y.~Liu, J.~Lee, M.~Park, S.~Kim, E.~Yang, S.~J. Hwang, and Y.~Yang, ``Learning
  to propagate labels: Transductive propagation network for few-shot
  learning,'' in \emph{ICLR}, 2018.

\bibitem{dhillon2019baseline}
G.~S. Dhillon, P.~Chaudhari, A.~Ravichandran, and S.~Soatto, ``A baseline for
  few-shot image classification,'' in \emph{ICLR}, 2019.

\bibitem{Mao2016Image}
X.~Mao, C.~Shen, and Y.-B. Yang, ``Image restoration using very deep
  convolutional encoder-decoder networks with symmetric skip connections,'' in
  \emph{NeurIPS}, 2017.

\bibitem{kingma2014adam}
D.~P. Kingma and J.~Ba, ``Adam: A method for stochastic optimization,'' in
  \emph{ICLR}, 2014.

\bibitem{luo2021few}
Q.~Luo, L.~Wang, J.~Lv, S.~Xiang, and C.~Pan, ``Few-shot learning via feature
  hallucination with variational inference,'' in \emph{WACV}, 2021, pp.
  3963--3972.

\bibitem{parnami2022learning}
A.~Parnami and M.~Lee, ``Learning from few examples: A summary of approaches to
  few-shot learning,'' \emph{arXiv preprint arXiv:2203.04291}, 2022.

\bibitem{hariharan2017low}
B.~Hariharan and R.~Girshick, ``Low-shot visual recognition by shrinking and
  hallucinating features,'' in \emph{CVPR}, 2017.

\bibitem{li2020adversarial}
K.~Li, Y.~Zhang, K.~Li, and Y.~Fu, ``Adversarial feature hallucination networks
  for few-shot learning,'' in \emph{CVPR}, 2020, pp. 13\,470--13\,479.

\bibitem{yao2019hierarchically}
H.~Yao, Y.~Wei, J.~Huang, and Z.~Li, ``Hierarchically structured
  meta-learning,'' in \emph{ICML}, 2019, pp. 7045--7054.

\bibitem{rusu2018meta}
A.~A. Rusu, D.~Rao, J.~Sygnowski, O.~Vinyals, R.~Pascanu, S.~Osindero, and
  R.~Hadsell, ``Meta-learning with latent embedding optimization,'' in
  \emph{International Conference on Learning Representations}, 2018.

\bibitem{Lee_2022_WACV}
S.~Lee, S.~Lee, and B.~C. Song, ``Contextual gradient scaling for few-shot
  learning,'' in \emph{WACV}, January 2022, pp. 834--843.

\bibitem{huang2021pseudo}
K.~Huang, J.~Geng, W.~Jiang, X.~Deng, and Z.~Xu, ``Pseudo-loss confidence
  metric for semi-supervised few-shot learning,'' in \emph{ICCV}, 2021, pp.
  8671--8680.

\bibitem{zhuang2018relationnet}
Y.~Zhuang, L.~Tao, F.~Yang, C.~Ma, Z.~Zhang, H.~Jia, and X.~Xie, ``Relationnet:
  Learning deep-aligned representation for semantic image segmentation,'' in
  \emph{ICPR}.\hskip 1em plus 0.5em minus 0.4em\relax IEEE, 2018, pp.
  1506--1511.

\bibitem{sung2018learning}
F.~Sung, Y.~Yang, L.~Zhang, T.~Xiang, P.~H. Torr, and T.~M. Hospedales,
  ``Learning to compare: Relation network for few-shot learning,'' in
  \emph{CVPR}, 2018, pp. 1199--1208.

\bibitem{Afrasiyabi_2022_CVPR}
A.~Afrasiyabi, H.~Larochelle, J.-F. Lalonde, and C.~Gagn\'e, ``Matching feature
  sets for few-shot image classification,'' in \emph{CVPR}, June 2022, pp.
  9014--9024.

\bibitem{chen2019closer}
W.-Y. Chen, Y.-C. Liu, Z.~Kira, Y.-C.~F. Wang, and J.-B. Huang, ``A closer look
  at few-shot classification,'' \emph{arXiv preprint arXiv:1904.04232}, 2019.

\bibitem{li2022cross}
W.-H. Li, X.~Liu, and H.~Bilen, ``Cross-domain few-shot learning with
  task-specific adapters,'' in \emph{CVPR}, 2022, pp. 7161--7170.

\bibitem{khalifa2021comprehensive}
N.~E. Khalifa, M.~Loey, and S.~Mirjalili, ``A comprehensive survey of recent
  trends in deep learning for digital images augmentation,'' \emph{Artificial
  Intelligence Review}, pp. 1--27, 2021.

\bibitem{xie2020self}
Q.~Xie, M.-T. Luong, E.~Hovy, and Q.~V. Le, ``Self-training with noisy student
  improves imagenet classification,'' in \emph{CVPR}, 2020, pp.
  10\,687--10\,698.

\bibitem{theagarajan2019shieldnets}
R.~Theagarajan, M.~Chen, B.~Bhanu, and J.~Zhang, ``Shieldnets: Defending
  against adversarial attacks using probabilistic adversarial robustness,'' in
  \emph{CVPR}, 2019, pp. 6988--6996.

\bibitem{wang2021zero}
W.~Wang, H.~Xu, G.~Wang, W.~Wang, and L.~Carin, ``Zero-shot recognition via
  optimal transport,'' in \emph{WACV}, 2021.

\bibitem{nguyen2021most}
T.~Nguyen, T.~Le, H.~Zhao, Q.~H. Tran, T.~Nguyen, and D.~Phung, ``Most:
  Multi-source domain adaptation via optimal transport for student-teacher
  learning,'' in \emph{UAI}.\hskip 1em plus 0.5em minus 0.4em\relax PMLR, 2021,
  pp. 225--235.

\end{thebibliography}

\begin{IEEEbiography}[{\includegraphics[width=1.1in,height=1.25in,clip,keepaspectratio]{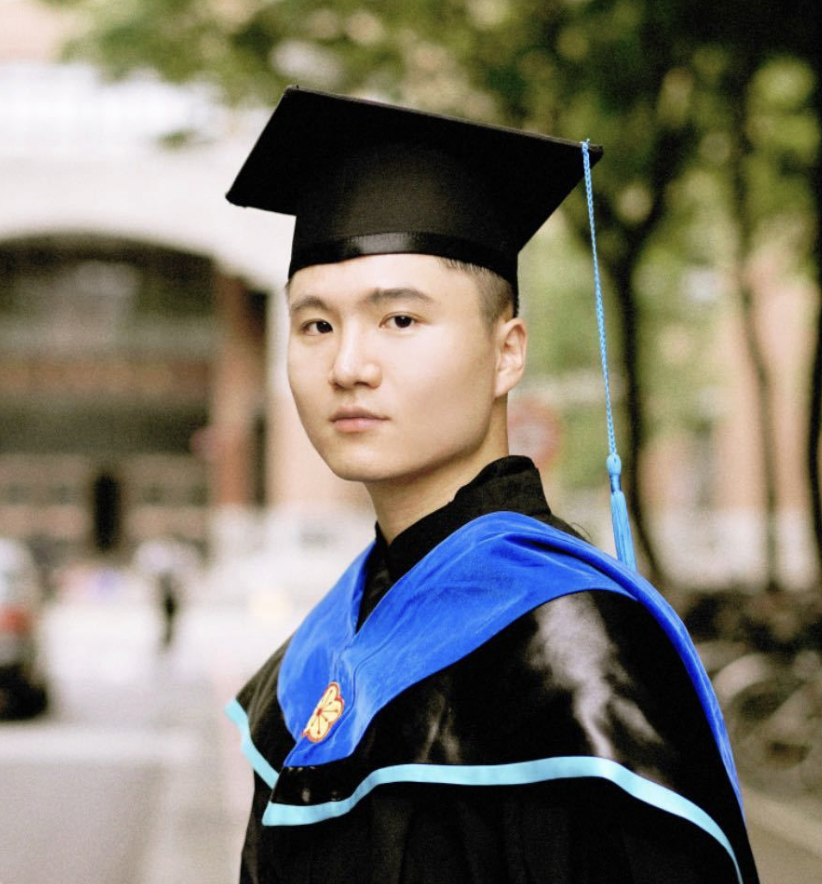}}]{Siyang Jiang} worked in Huizhou University (HZU). He received his master's degree from the National Taiwan University (NTU), Taipei, Taiwan, in 2021. His current research interests include deep learning, machine learning, few-shot learning, and meta-learning. He was the recipient of the Best Student Paper Award at PAKDD 2022.
\end{IEEEbiography}

\begin{IEEEbiography}
[{\includegraphics[width=1.25in,height=1.25in,clip,keepaspectratio]{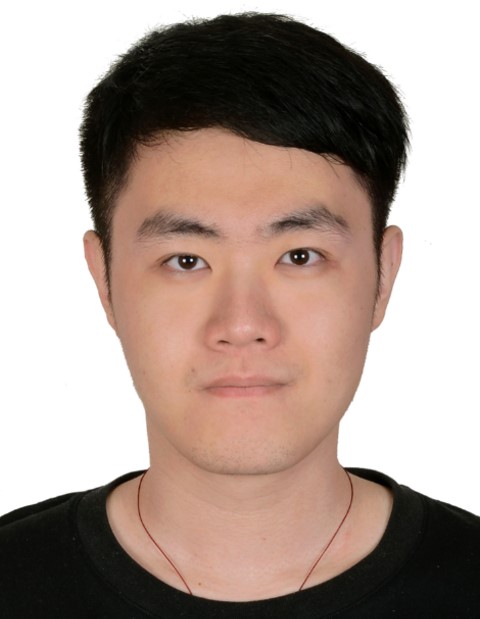}}]
{Rui Fang} is currently working toward his PhD degree at the National Taiwan University (NTU), Taipei, Taiwan. He received his master degree from the National Taipei University of Technology (NTUT), Taipei. in 2021. His current research interests include deep learning, few-shot learning, and deep learning acceleration.
\end{IEEEbiography}

\begin{IEEEbiography}[{\includegraphics[width=1.25in,height=1.25in,clip,keepaspectratio]{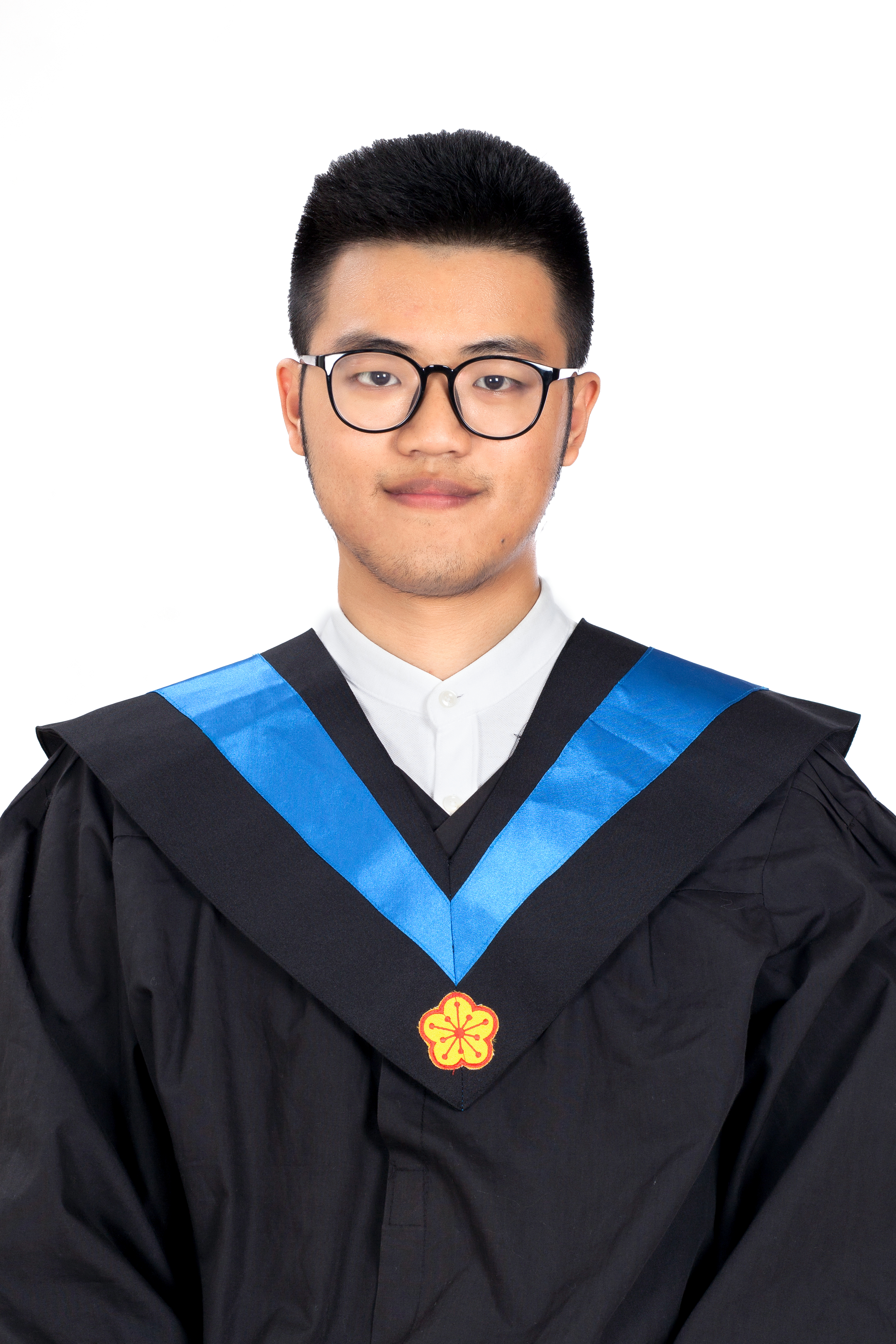}}]{Hsi-Wen Chen} currently working toward the PhD
degree at the National Taiwan University (NTU), Taipei. He received the B.S. degree and M.S. degree both from NTU. His research field includes deep learning, data mining, reinforcement learning, and few-shot learning. He was the recipient of the National Outstanding Doctor Scholarship in NTU.
\end{IEEEbiography}

\begin{IEEEbiography}
[{\includegraphics[width=1.25in,height=1.25in,clip,keepaspectratio]{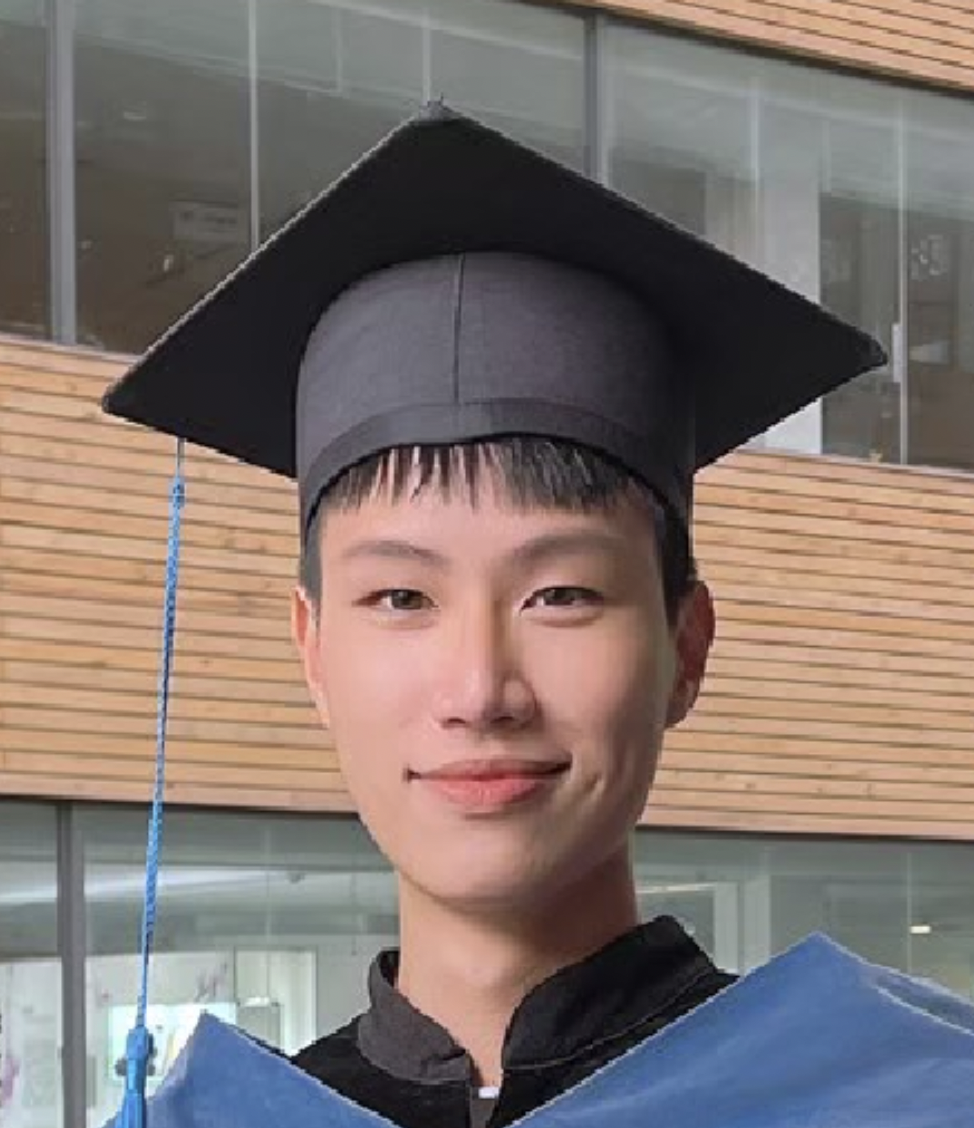}}]{Wei Ding} received the B.S. degree in computer science from National Tsing Hua University, Hsinchu, and the M.S. degree in electrical engineering from National Taiwan University, Taipei. His research field includes deep learning, reinforcement learning, and few-shot learning.
He was the recipient of the Best Student Paper Award at PAKDD 2022.
\end{IEEEbiography}

\begin{IEEEbiography}[{\includegraphics[width=1.25in,height=1.25in,clip,keepaspectratio]{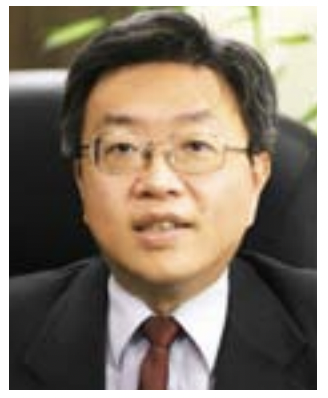}}]{Ming-Syan Chen} (Fellow, IEEE/ACM) received the B.S. degree in electrical engineering from National Taiwan University, Taipei, and the M.S. and Ph.D. degrees in computer, information, and control engineering from the University of Michigan at Ann Arbor, Ann Arbor, in 1985 and 1988, respectively. He was a Research Staff Member with the IBM Thomas J. Watson Research Center, Yorktown Heights, NY, USA, the President/CEO of Institute for Information Industry (III), and the Director of the Research Center of Information Technology Innovation (CITI), Academia Sinica. He is currently a Distinguished Professor with the EE Department, National Taiwan University. His research interests include databases, data mining, cloud computing, and multimedia networking. 
He was a recipient of the Academic Award of the Ministry of Education, the National Science Council (NSC) Distinguished Research Award for his research work, and the Outstanding Innovation Award from IBM Corporate for his contribution to a major database product.
\end{IEEEbiography}

\end{document}